\documentclass[twoside]{article}
\usepackage[accepted]{aistats2025}

\usepackage[utf8]{inputenc} 
\usepackage[T1]{fontenc}    
\usepackage{hyperref}       
\usepackage{url}            
\usepackage{booktabs}       
\usepackage{amsfonts}       
\usepackage{nicefrac}       
\usepackage{microtype}      
\usepackage{xcolor}         
\usepackage{amsmath}
\usepackage{graphicx}
\usepackage{epstopdf}
\usepackage{amssymb}
\usepackage{mathtools}
\usepackage{amsthm}
\usepackage{float}
\usepackage{blkarray}
\usepackage{blkarray, bigstrut}
\usepackage{xparse}
\usepackage{algorithm}
\usepackage{algpseudocode}
\usepackage{subfigure}
\usepackage{natbib}
\usepackage{stfloats}
\newcommand{\matindex}[1]{\mbox{\scriptsize#1}}

\newtheorem{theorem}{Theorem}[section]

\newtheorem{lemma}[theorem]{Lemma}

\newtheorem{definition}[theorem]{Definition}

\def \Pa {\mathrm{PA}}
\def \pa {\mathrm{pa}}
\def \SPA {\mathrm{SPA}}
\def \spa {\mathrm{spa}}
\def \nw {n_\mathrm{w}}
\def \bs {b_\mathrm{st}}
\def \aw {a_\mathrm{w}}
\def \ns {n_\mathrm{st}}
\def \tsub {T_\mathrm{sub}}
\newcommand{\tub}{\tau_\text{ub}}
\newcommand{\Xb}{\mathbf{X}}
\newcommand{\indep }{\perp\!\!\!\perp}

\def \pind {\mathrm{pInd}}
\newcommand{\appendixhead}%
{\textbf{\huge Appendix}}
\newcommand{\aistatsappendixtitle}[1]{%
  \hsize\textwidth
  \linewidth\hsize
  \toptitlebar{\centering{\Large\bfseries #1 \par}}
  \bottomtitlebar
}

\begin{document}

\twocolumn[
\aistatstitle{Causal Discovery-Driven Change Point Detection in Time Series}

\aistatsauthor{Shanyun Gao \And Raghavendra Addanki \And  Tong Yu\And Ryan A. Rossi \And Murat Kocaoglu} 

\aistatsaddress{Purdue University \And  Adobe Research \And Adobe Research  \And Adobe Research  \And Purdue University }]

\begin{abstract}
Change point detection in time series aims to identify moments when the probability distribution of time series changes. It is widely applied in many areas, such as human activity sensing and medical science. In the context of multivariate time series, this typically involves examining the joint distribution of multiple variables: If the distribution of any one variable changes, the entire time series undergoes a distribution shift. However, in practical applications, we may be interested only in certain components of the time series, exploring abrupt changes in their distributions while accounting for the presence of other components. Here, assuming an underlying structural causal model that governs the time-series data generation, we address this task by proposing a two-stage non-parametric algorithm that first learns parts of the causal structure through constraint-based discovery methods, and then employs conditional relative Pearson divergence estimation to identify the change points. The conditional relative Pearson divergence quantifies the distribution difference between consecutive segments in the time series, while the causal discovery method allows a focus on the causal mechanism, facilitating access to independent and identically distributed (IID) samples. Theoretically, the typical assumption of samples being IID in conventional change point detection methods can be relaxed based on the Causal Markov Condition. Through experiments on both synthetic and real-world datasets, we validate the correctness and utility of our approach.

\end{abstract}

\section{INTRODUCTION}\label{sec:intro}
Change point analysis aims to detect distribution shifts in observational time series data. This topic has been explored extensively and is popular in many areas, such as human activity analysis \citep{brahim2004gaussian,cleland2014evaluation}, image analysis \citep{radke2005image} and financial markets \citep{talih2005structural}. 

Traditionally, change point methods focus on detecting shifts in the \textbf{joint distribution} of all variables in the time series. There is an implicit assumption that the entire joint distribution of all the variables significantly shifts whenever a change occurs in any of the variables in the multivariate time series. This assumption works well if we only care about broad, overarching changes, but it can be restrictive and overlook an important real-world consideration: in many scenarios, we are more concerned with local changes rather than global shifts. Motivated by the causal invariance principle and invoking a causal modeling perspective, we are interested in detecting changes that specifically affect the \textbf{causal mechanisms} governing the variables in the time series.


In the field of human health and medicine, researchers aim at identifying subtle changes in specific patient conditions before the overall health deteriorates significantly, such as capturing antecedent signs and symptoms of sepsis in \citep{shashikumar2017early} and \citep{goh2021artificial}. Similarly, detecting signal changes within vast datasets to predict anomalies before the onset of overall financial distress is a critical focus in finance, as demonstrated in \citep{koyuncugil2012financial} and \citep{kliestik2018bankruptcy}. Therefore, it is essential to shift our focus from global joint distribution changes to causal mechanism changes in practical applications.


Additionally, from a theoretical perspective, many existing methods, including  ~\citep{aminikhanghahi2017survey,harchaoui2008kernel,liu2013change,saggioro2020reconstructing}, require independent and identically distributed (IID) samples. Although many of these methods are robust 
to non-IID samples in simulated data~
such as \citep{liu2013change}, they lack theoretical guarantees, which are needed for trustworthy deployment of these algorithms in the real world, especially in safety-critical applications as in~\citep{liu2018change}.

Driven by both practical needs and theoretical importance, we propose a novel change point detection algorithm 
that integrates change point detection with causal discovery.
First, causal structure helps us obtain a more fine-grained look at the joint distribution by 
disentangling the causal mechanism of each component in the time series. 
Second, relying on the Causal Markov Condition assumption in the Structural Causal Model (SCM) framework of~\citep{pearl2009causality}, the correlated samples become independent and identically distributed (IID) when conditioned on their causal parents. 
In this context, our change point detection method aims to identify shifts in the \textbf{causal mechanisms}, specifically, the conditional distribution of the target variable given its parents. 


In this paper, we summarize our main contributions as follows:
\begin{itemize}
    \item We propose a novel non-parametric algorithm, called Causal-RuLSIF,  for detecting change points in causal mechanisms within discrete-valued time series data. The algorithm assumes an underlying \emph{Mechanism-Shift} SCM but does not impose constraints on the causal mechanisms or the data distributions. Our algorithm introduces a novel dynamic RuLSIF estimator for detecting change points, on the basis of \citep{yamada2013relative} and \citep{liu2013change}, and integrates the PCMCI algorithm for causal discovery from \citep{runge2019detecting}. 
    \item We validate our method with synthetic simulations on both \emph{soft mechanism change} and \emph{hard mechanism change} time series, demonstrating that it reliably detects change points with high probability. We also employ our method in an air pollution application. 
\end{itemize}

      \begin{figure*}[t]
        \centering
        \includegraphics[width=1\linewidth]{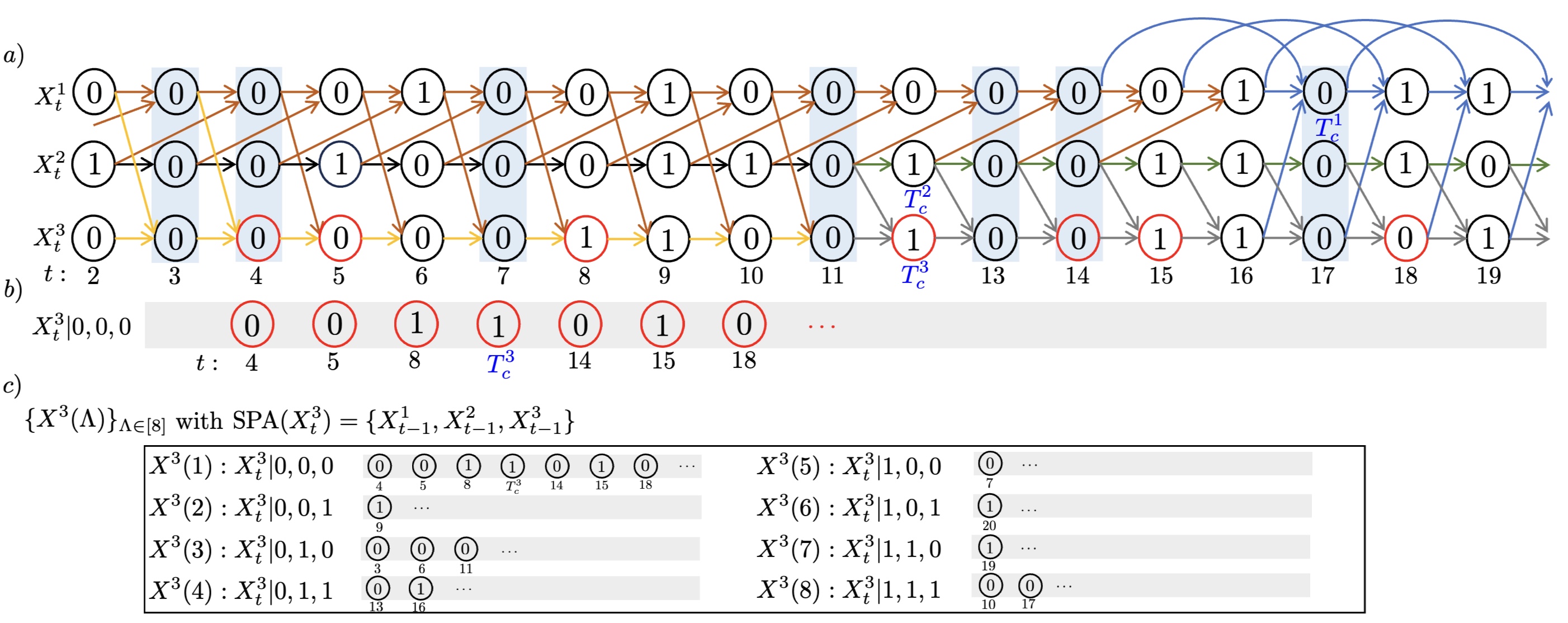}
        \caption{ 
         An illustration of collecting time series segments (Def.\ref{def:TimeSeriesSegments}) in Causal-RuLSIF. a). The Causal Graph of a 3-variate binary time series $V=\{\Xb^{1}, \Xb^{2}, \Xb^{3}\}$ with a \emph{Mechanism-Shift} SCM (Def.\ref{def:Mechanism-Shift}), where $c^1=c^2=c^3=1$, indicating one change point in each univariate time series. Note that the location of change points can differ across univariate time series. E.g., the change point for $X^j$, where $j\in[3]$, is denoted as $T^j_{c}$, and in this case $T^1_{c}\neq T^2_{c}=T^3_{c}$. The edges of different colors represent distinct causal mechanisms. E.g., for $\Xb^1$,  the causal mechanism indicated by the brown edges shifts to a different causal mechanism represented by the blue edges at the time point $T^1_{c}$. There are no restrictions on the causal mechanisms. E.g., for $\Xb^1$, the two causal mechanisms before and after $T^1_{c}$ differ in both the parent sets and the deterministic function $f_{1,t}(\cdot)$; we refer to this scenario as \emph{hard mechanism change}. In contrast, for $\Xb^2$, the two causal mechanisms differ only in $f_{2,t}(\cdot)$ while the parent set remains invariant over time; we refer to this situation as \emph{soft mechanism change}. b). Construction of a time series segment for $\Xb^3$. With $\text{PA}(X^3_{t<T^3_{c}})=\{X^1_{t-1},X^3_{t-1}\}$, $\text{PA}(X^3_{t\geq T^3_{c}})=\{X^2_{t-1},X^3_{t-1}\}$, leading to $\text{SPA}(X^{3}_{t})=\text{PA}(X^3_{t<T^3_{c}})\cup \text{PA}(X^3_{t\geq T^3_{c}})$ for all $t$ (Def.\ref{def: Illusory Parent Sets}). Given that $|\text{SPA}(X^{3}_{t})|=3$ and the domain size of binary variables is 2, one specific configuration of $\text{SPA}(X^{3}_{t})$ is $\{0,0,0\}$. To form the corresponding time series segment, we need to collect all variables $X^3_t\in\Xb^3$  for which $t$ satisfies $\text{spa}(X^{3}_{t})=\{X^1_{t-1}=0,X^2_{t-1}=0,X^3_{t-1}=0\}$. c). Collection of 8 time series segments for $\Xb^3$ as there are total $8$ parents' configurations, starting from $\{0,0,0\}$, $\{0,0,1\}$, $\{0,1,0\}$, and continuing to $\{1,1,1\}$.}
        \label{full causal graph}
    \end{figure*}

\section{RELATED WORK}\label{sec:related-work}

There are various change point detection methods based on the compatibility between the statistical feature shifts they detect and the characteristics of the data. 
Likelihood ratio methods compare the probability density of two consecutive intervals of the time series data and the significant difference in the probability density implies the change point. One representative estimator is the relative unconstrained least-squares importance fitting (RuLSIF) in \citep{yamada2013relative} and \citep{liu2013change}. 
Kernel-based methods in \citep{harchaoui2008kernel} and \citep{harchaoui2009regularized} utilize kernel-based test statistics to test the homogeneity of sliding windows in time series data. Other methods include Probabilistic methods, Graph-based methods, clustering methods and Subspace modeling. See \citep{chib1998estimation}, \citep{friedman1979multivariate}, \citep{keogh2001online} and \citep{itoh2010change}, respectively. See \citep{aminikhanghahi2017survey} for a survey. As discussed in the instruction section, such existing methods concentrate on joint features associated with the whole time series, with many assuming IID samples. 

In addition to these methods, some change point detection techniques specifically target different challenges or leverage additional information embedded in the dataset. \citep{cho2015multiple, barigozzi2018simultaneous, kovacs2023seeded} introduced methods for detecting multiple change points in high-dimensional time series. \citep{killick2012optimal, maidstone2017optimal} focus on increasing the computational efficiency. \citep{safikhani2022joint} proposed a method for estimating structural change points and parameters in high-dimensional piecewise VAR models. \citep{qiu2012granger} applied Granger causality for anomaly detection in time series data. \citep{bardet2010detecting, diop2022epidemic} detect change points in a large class of causal time series models including AR($\infty$), ARCH($\infty$) and TARCH($\infty$) models. More recently, \citep{huang2024causal} introduced a change point detection method within the context of a response variable $Y$ and covariates $X$, assuming a linear SCM that generates $Y$ from $X$.

To the best of our knowledge, no other non-parametric change point detection models are capable of identifying shifts in causal mechanisms without imposing constraints on the form of the mechanism.

Additionally, there is a noteworthy "bonus" contribution to the field of causal discovery. Our change point detection algorithm, when integrated with a postprocessing causal discovery step, is, to our knowledge, the first non-parametric method capable of learning sudden-shift causal structures from non-stationary time series data without assuming strict periodicity as in \citep{gao2023causal}. More related work on this contribution can be found in Appendix~\ref{app:related work}.

 \section{CAUSAL-\textsc{RuLSIF}: DETECTING CHANGE POINT IN A CAUSAL TIME SERIES}
\label{submission}
In this section, we present the framework and problem formulation for change point detection in causal time series, including the key assumptions.

\subsection{Preliminaries}

Let $\mathcal{G}(V, E)$ denote the underlying causal graph. The set of all incoming neighbors for each variable $X \in V$ is defined as the parent set, denoted by $\Pa(X)$. For any $X, Y \in V$ and $S \subset V$, we denote the conditional independence: $X$ is independent of $Y$ conditioned on $S$, by $X \indep Y \mid S$.

For simplicity, let's define sets: $[b]:=\{1,2,...,b\}$ and $[a,b]:=\{a,a+1,...,b\}$, where $a,b\in \mathbb{N}$.
Let $X^j_{t} \in \mathbb{R}$ represent the variable of $j$th \emph{component} univariate time series at time $t$; $\Xb^{j}=\{X^j_{t}\}_{t\in [T]}\in\mathbb{R}^{T}$ denote a \emph{component} univariate time series which is a component in multivariate time series $V$ and $\Xb_{t}=\{X^j_{t}\}_{j\in [n]}\in \mathbb{R}^{n}$ denote a slice of all variables at time point $t$. Note that $V=\{\Xb^{j}\}_{j\in [n]}=\{\Xb_{t}\}_{t\in [T]}\in \mathbb{R}^{n\times T}$ represents a $n$-variate time series, where it can be interpreted either as a collection of $n$ univariate time series (viewed horizontally) or as a sequence of time slices across all variables (viewed vertically). By default, we assume $n>1$ and hence $\Xb^{j}\subsetneq V$, and $p(V)\neq0$, where $p(.)$ denotes the probability. For discrete-valued time series $V$, we assume that the domain of each component $\Xb^j\subsetneq V$, denoted by $D=\{d_1,\cdots,d_s\}$, is the same, that is, $X^j_t \in D $ for all $j \in [n]$ and $t \in [T]$. 

In this paper, the subscript notation over $\{\cdot\}$ indicates that the elements within $\{\cdot\}$ are aggregated across all values of the subscripts.

As $\Pa(X^j_t)$ represents a set containing random variables, $\pa(X^j_t)$ refers to one specific configuration. Consider $\Pa(X^j_{t})=\{X^{i_1}_{t_1},X^{i_2}_{t_2},\cdots\}$ where $1\leq i_1\leq i_2 \leq \cdots \leq n $ and $T \ge t_1\ge t_2 \ge \cdots \ge 1 $. E.g., $\pa(X)=\{1,0,1,2,\cdots\}$ corresponds to a particular configuration (or instantiation). We assume that the configuration set is ordered by the parent set $\{X^{i_1}_{t_1}, X^{i_2}_{t_2},\cdots\}$, prioritizing ascending variable indices over descending time indices.

We start by defining a type of Structural Causal Model (SCM) that captures our setting.
\begin{definition}[\textit{Mechanism-Shift} SCM]\label{def:Mechanism-Shift}
   A Mechanism-Shift SCM is a tuple  $\mathcal{M}=\langle V,\mathcal{F},\mathcal{E},\mathbb{P} \rangle$ where there exists a $\tau_{\max} \in\mathbb{N}^{+}$, defined as: $\tau_{\max}\coloneqq \max\{\tau:X^{i}_{t-\tau}\in \Pa(X^{j}_{t}),i,j \in [n], t\in[T]\}$, such that each variable $X^{j}_{t>\tau_{\max}}\in V$ is a deterministic function of its parent set $\Pa(X^{j}_{t>\tau_{\max}})\in V$ and an unobserved (exogenous) variable $\epsilon^{j}_{t>\tau_{\max}}\in \mathcal{E}$:
\begin{align}
X^{j}_{t}&=f_{j, t}(\Pa(X^{j}_{t}),\epsilon^{j}_{t}), \;j \in [n], t \in [\tau_{\max}+1, T], \label{definition_of_general_SCM}
\end{align}
and for each $j \in [n]$, there exists an ascending sequence of time points $\{T^{j}_{1},T^{j}_{2},\cdots,T^{j}_{c^j} \}$ with $c^j \in \mathbb{N}^{+},\tau_{\max}<T^{j}_{1}<\cdots<T^{j}_{c^j}<T$ such that:
   \begin{align}
   &a) \;f_{j, t_1} = f_{j, t_2}, \text{ if }\forall c \in [c^j]\text{ s.t. }  T^j_c \notin [t_1, t_2];\label{a}\\
     &b) \;f_{j, t_1} \neq f_{j, t_2}, \notag\\&\text{ \;\;if }\exists c \in [c^j]\text{ s.t. }T^j_{c-1}\leq t_1< T^j_c\leq t_2<T^j_{c+1};\label{b}\\
   &c) \;\Pa(X^{j}_{t_1})=\{X^{i}_{t_1-s}:X^{i}_{t_2-s}\in \Pa(X^{j}_{t_2}), i\in[n]\}, \notag\\&\text{\;\;\;if }\forall c \in [c^j]\text{ s.t. }  T^j_c \notin [t_1,t_2];\label{c}\\
   &d)\;\epsilon^j_t \text{ are i.i.d. } \forall t\in[T].\label{d}
   \end{align}
    are satisfied for all $ t_1,t_2 \in [\tau_{\max}+1, T]$, where $f_{j, t}, f_{j, t_{1}}, f_{j, t_{2}}\in \mathcal{F}$ and $\{\epsilon^{j}_{t}\}_{t\in[T]}$ are jointly independent with probability measure $\mathbb{P}$. $\tau_{\max}$ is the finite maximal lag in the causal graph $\mathcal{G}$. Define $T^j_0\coloneqq \tau_{\max}$ and $T^j_{c^j+1}=T$.
    \end{definition}
    
    This indicates that within the univariate time series $\Xb^{j}$ in $V$, there is a finite number of change points, denoted by $c^j$. The variables in $\Xb^{j}$ before and after each change point should exhibit distinct causal mechanisms, as illustrated in b), without overlapping with other change points. Two variables in $\Xb^{j}$  that do not span any change points should share the same function and time-shift invariant parents, as depicted in a) and c). An instance of this model is shown in Fig.~\ref{full causal graph}a).

\begin{definition}[\textit{Illusory Parent Sets}]\label{def: Illusory Parent Sets} 
   For a univariate time series $\Xb^{j}\in V$ with Mechanism-Shift SCM having change points $\{T^{j}_{1},T^{j}_{2},\cdots,T^{j}_{c^j} \}$ with $c^j \in \mathbb{N}^{+},\tau_{\max}<T^{j}_{1}<\cdots<T^{j}_{c^j}<T$, parent set index $\pind^j_{k\in [c^{j}+1]}$ is defined as:
   \begin{align}
   &\pind^{j}_{k}\coloneqq\{(\tau_{i},y_{i})\}_{i\in [m]}, \text{ given }\notag \\&\Pa(X^{j}_{t})=\{X^{y_1}_{t-\tau_1}, X^{y_2}_{t-\tau_2},...,X^{y_{m}}_{t-\tau_{m}}\},
   \end{align}
   for all $t \in [T^j_{k-1},T^j_k]$, where $m=|\Pa(X^{j}_t))|$, $\tau_{i}$ is the time lag and $y_{i}$ is the variable index; set $T^j_{0}=\tau_\text{max}$ and $T^j_{c^j+1}=T$.
   Given $\pind^j_{k}$, \textit{Illusory Parent Sets} are defined as:
   \begin{align}
       &\Pa_{k}(X^{j}_{t})=\bigl\{X^{y_{i}}_{t-\tau_i}:(\tau_i, y_i)\in \pind^{j}_{k}\bigr\},\notag\\&\forall k \in \{k:t \notin [T^j_{k-1},T^j_k]\}\label{fake_parent}
   \end{align}
\end{definition}
Essentially, the illusory parent sets of $X^{j}_{t}$ are time-shifted versions of the parent sets of other variables in $\Xb^{j}$ that exhibit different causal mechanisms than $X^{j}_{t}$ across change points. These sets generalize a concept described in \citep{gao2023causal}. There should be $c^j$ illusory parent sets for each variable $X^j_t$. 

For simplicity, we extend \textit{Illusory Parent Sets} to also include the true parent set of $X^{j}_{t}$, resulting in each variable $X^j_t$ having $c^j+1$ illusory parent sets, one of which is $\Pa(X^j_t)$.
The term \textit{illusory} now indicates that these parent sets \textit{may} not exist for $t$  but must be valid for some $t'\in[T]$.

Further, we define the union parent set as:
\begin{align}
  \SPA(X^j_t)\coloneqq \cup_{k\in[c^j+1]}\Pa_k(X^j_t), t\in[\tau_\text{max}+1,T]  
\end{align}
E.g., in Fig.~\ref{full causal graph}a), $\Pa(X^1_t)=\{X^1_{t-1},X^2_{t-2}\}$ with $t<T^1_{cp}$ and $\Pa(X^1_t)=\{X^1_{t-1},X^1_{t-3}, X^3_{t-1}\}$ with $t\geq T^1_{cp}$. As $c^1=1$, we have two illusory parent sets with index $k\in [c^1+1]$. Specifically, the two illusory parent sets of $\Xb^1$ are $\Pa_1(X^1_t)=\{X^1_{t-1},X^2_{t-2}\}$ and $\Pa_2(X^1_t)=\{X^1_{t-1},X^1_{t-3}, X^3_{t-1}\}$, leading to $\SPA(X^1_t)=\{X^1_{t-1},X^1_{t-3},X^2_{t-2}, X^3_{t-1}\}$.
\begin{definition}[\textit{time series segments}]\label{def:TimeSeriesSegments} 
For a univariate discrete-valued time series $\Xb^{j}\in V$ with Mechanism-Shift SCM having change points set $\{T^{j}_{1},T^{j}_{2},\cdots,T^{j}_{c^j} \}$ with $c^j \in \mathbb{N}^{+},\tau_{\max}<T^{j}_{1}<\cdots<T^{j}_{c^j}<T$, and finite domain set $D=\{d_1,d_2,\cdots,d_s\}$, the time series segments are a collection of non-overlapping non-empty subsets $ \{X^j(\Lambda)\}_{\Lambda\in[s^{|\SPA(X^j_t)|}]}$ such that:
\begin{align}
    &X^j(\Lambda)\coloneqq \{X^j_t: t\in [\tau_\text{max}+1,T],\pa(X^j_t)=\Sigma^j_\Lambda\}, \notag\\&\Lambda\in[s^{|\SPA(X^j_t)|}].
\end{align}
Here, $\Sigma$ represents the configuration matrix of $\SPA(X^j_t)$, with each row representing a unique configuration of $\SPA(X^j_t)$. $\Lambda$ refers to the row index of one specific configuration in $\Sigma$. Since the domain size is $|D|=s$, $\Sigma$ has $s^{|\SPA(X^j_t)|}$ rows, representing all possible configurations of $\SPA(X^j_t)$.
Based on the above definition, we can observe that:
\begin{align}
    &\bigcup\limits_{\Lambda\in [s^{|\SPA(X^j_t)|}]} \{ X^j(\Lambda) \} = \Xb^j,
    \\&\bigcap\limits_{\Lambda_1 \neq \Lambda_2\in[s^{|\SPA(X^j_t)|}]]} \{X^j(\Lambda_1),X^j(\Lambda_2)\} = \emptyset.
\end{align}
\end{definition}

In summary, the time series segments $\{\Xb^j(\Lambda)\}$ partition $\Xb^j$ into multiple non-overlapping, non-empty sub-time series, conditioned on the configurations of the union parent set $\SPA(X^j_t)$. Variables $X^j_t$ within the same $\Xb^j(\Lambda)$ share identical configurations of $\SPA(X^j_t)$. In Fig.~\ref{full causal graph}b) and \ref{full causal graph}c), $\Xb^3$ is divided into 8 time series segments, with each segment having the same configuration $\text{spa}(X^3_t)$. Refer to Appendix~\ref{app:tss} for a detailed example.

\noindent \textbf{RuLSIF: Robust Distributional Distance Estimation}. Given two distributions $p(\cdot)$ and $p'(\cdot)$ defined over the same support, there exist many metrics measuring the distance between them. In~\citep{yamada2013relative}, the authors proposed a method named RuLSIF, with an $\alpha$-relative divergence estimation $r_\alpha(x)$ and a corresponding metric, $\alpha$-relative Pearson Divergence $PE_\alpha$ defined as following:
\begin{align}
    &r_\alpha(x)\coloneqq \frac{p(x)}{(1-\alpha)p(x)+\alpha p'(x)}\coloneqq\frac{p(x)}{q_{\alpha}(x)}\\ &PE_{\alpha}\coloneqq \frac{1}{2}\mathbb{E}_{x \sim q_\alpha}[(r_\alpha(x)-1)^2]
\end{align}
where $\alpha$ is a parameter used to bound $r_\alpha(x)$.

In our work, we innovatively introduce a dynamic parameter, $\beta_i$, within the shifting window framework, rather than relying solely on a fixed hyperparameter $\alpha$ in $PE_\alpha$. The parameter $\beta_i$ represents the \textit{concentration} of the distribution $p'(x)$ within a mixture distribution, where $i$ denotes the window index.

As the shifting window moves from $t=0$ to $t=T$ with a window size of $2\nw$ and a stride size of $\ns$, the window index $i$ increases. Notably, $\beta_i$ transitions from $0$ to $1$ when the shifting window crosses a change point. Specifically, if a change point occurs in the second half of the window, the samples in that half are drawn from a mixture distribution represented by $(1-\beta_i)p(x) + \beta_i p'(x)$. Here, $\beta_i$ indicates the proportion of samples in the second half derived from $p'(x)$, reflecting the change in the underlying distribution of data as the window shifts.

Based on this, the dynamic density ratio and the corresponding Pearson Divergence (PE) score are defined as follows:
\begin{align}
  &r_{\alpha\beta_i}(x)\coloneqq \frac{p(x)}{(1-\alpha\beta_i)p(x)+\alpha\beta_i p'(x)}\coloneqq\frac{p(x)}{q_{\alpha\beta_i}(x)} \\
  &PE_{\alpha\beta_i}\coloneqq \frac{1}{2}\mathbb{E}_{x \sim q_{\alpha\beta_i}}[(r_{\alpha\beta_i}(x)-1)^2]
\end{align}
Further details on dynamic RuLSIF, including the rationale for its selection as the metric and the official definition of $\beta_i$ are provided in Appendix~\ref{app:RuLSIF}.

\noindent \textbf{Assumptions}. We highlight three key assumptions here: (1) For any univariate time series, there is at most one change point, or if multiple change points exist, the temporal distance between any two must be at least $\Delta_c$. (2) There is a lower bound on the dynamic PE divergence. (3) The number of change points is known. The first assumption ensures that multiple change points do not have a cancellation effect making it difficult to identify the change point. The second assumption is commonly applied in distributional distance estimation, particularly in finite sample and discrete settings. The third assumption offers theoretical guarantees but can be relaxed in practice, as demonstrated by the experimental results. The formalization of these three assumptions is provided in Appendix~\ref{app:assumptions}
, which also includes other assumptions established in prior works.


\section{CAUSAL-\textsc{RuLSIF} ALGORITHM}\label{sec:algorithms}
In this section, we first introduce an algorithm named Causal-RuLSIF. We then demonstrate the correctness of Causal-RuLSIF in accurately estimating the change point within a specified confidence interval. Finally, we provide a computational complexity analysis, highlighting both the contribution of the algorithm and its limitations.

\textbf{Overview of Algorithm \ref{alg:Causal-RuLSIF_brief} Causal-RuLSIF}: 
Please note that in this section, we assume a maximum of one change point per univariate time series; however, the framework can be easily extended to accommodate multiple change points. The parameters $\nw$, $\ns$, and $\alpha$ are set at the beginning. The value of $\tau_\text{ub}$ establishes the upper bound for the search space of $\tau_{\max}$ in PCMCI. Each component $\Xb^j$, where \( j \in [n] \), is analyzed sequentially.

Using PCMCI \citep{runge2019detecting}, we obtain $\widehat{\SPA}(X^j_t)$ for each variable \( X^j_t \in \Xb^j \) (line 2). To ensure balanced samples as required in Theorem \ref{pcmci}, PCMCI is applied to non-overlapping consecutive intervals, and the edges obtained are collected. Based on the parent configurations of $\widehat{\SPA}(X^j_t)$, we construct time series segments denoted as $\{X^j(\Lambda)\}_{\Lambda}$.

Next, we apply dynamic RuLSIF on sliding windows over $\{X^j(\Lambda)\}_{\Lambda}$, resulting in the divergence series represented by $\{\widehat{PE}^{j,\Lambda}_{\alpha\beta_i}\}_{i,\Lambda}$ (lines 6-8). Note that the number of divergence series corresponds to the number of time series segments. The change point estimator, denoted as $\widehat{T}_{c}^j$, is the window index \( i \) that maximizes $\{\widehat{PE}^{j,\Lambda}_{\alpha\beta_i}\}_{i,\Lambda}$ (lines 9-10). 

Since $\widehat{T}_{c}^j$ represents the window index within the time series segments rather than the original time index in $\Xb^j$, it must be projected back to $\Xb^j$. The final change point estimator is then defined as $\widetilde{T}_{c}^j$ (line 11). After obtaining an accurate estimator $\widehat{T}_{c}^j$, as established in Theorem~\ref{main: theorem 4.2}, one may optionally run PCMCI on the samples \( {X^j_t} \) before and after the change point \( \widetilde{T}_{c}^j \) to fully learn the underlying causal graph. 


\begin{algorithm}[t!]
\caption{Causal-RuLSIF} 
\label{alg:Causal-RuLSIF_brief}
\begin{algorithmic}[1]
\State \textbf{Input:} A $n$-variate time series $V = (\Xb^{1}, \cdots,\Xb^{n})$ with domain set $D=\{d_1,...d_s\}$. Set appropriate $\tub$, $\nw$, $\ns$ and $\alpha$.
\State A superset of the parent set is obtained using PCMCI with $\tub$ and denote it by $\widehat{\SPA}(X^j_t) \ \forall j, t$.
\For{$\Xb^j$ where $j \in [n]$} 
\State $\widehat{\Pa}(X^j_t)\leftarrow \widehat{\SPA}(X^j_t)$
\State Construct $\{X^j(\Lambda)\}$ based on configurations of $ \widehat{\SPA}(X^j_t)$.
\For{$X^j(\Lambda)$ where $\Lambda\in [s^{|\widehat{\SPA}(X^j_t)|}]$} 
\State Store divergence score series $\{\widehat{PE}^{j,\Lambda}_{\alpha\beta_i}\}$ with RuLSIF on sliding windows shifting over $X^j(\Lambda)$ where $i\in [\lfloor\frac{\tsub-2\nw}{\ns}\rfloor+1]$ and $\tsub\coloneqq |X^j(\Lambda)|$.
\EndFor
\State $\widehat{\Lambda}\leftarrow \arg_\Lambda\max\{\widehat{PE}^{j,\Lambda}_{\alpha\beta_i}\}_{i,\Lambda}\;\;\;\rhd$ Pick the $\Lambda$th time series segments whose maximum value over all segments is maximum.
\State $\widehat{T}_{c}^j\leftarrow \arg_i\max\{\widehat{PE}^{j,\widehat{\Lambda}}_{\alpha\beta_i}\}_i\;\;\;\;\;\rhd$ Pick the $i$th window index with maximum PE score on $X^j(\widehat{\Lambda})$.
\State Project $\widehat{T}_{c}^j$, where $j \in [n]$, back to the original time series $\Xb^j$ with $\widetilde{T}^j_{c} = \frac{1}{2}(t_{\widehat{T}_{c}^j} + t_{\widehat{T}_{c}^j+1})$, where $t_{\widehat{T}_{c}^j}$ is the time index in $\Xb$ corresponding to the midpoint of the window indexed by $\widehat{T}_{c}^j$ in $X^j(\widehat{\Lambda})$.

\State Consider $X^j_{t-\tau} \in \widehat{\Pa}(X^j_t)$. Remove $X^j_{t-\tau}$ from $\widehat{\Pa}(X^j_t)$ if $X^{i}_{t - \tau} \indep X^j_t \mid  \left(\widehat{S\Pa}(X^j_t) \cup  \widehat{S\Pa}(X^i_{t-\tau})\right) \setminus X^i_{t-\tau}$ on samples $\{X^j_t\}_{t<\widetilde{T}^j_{c}}$ and $\{X^j_t\}_{t\geq \widetilde{T}^j_{c}}$ respectively.
\EndFor
\State \Return{$\widetilde{T}^j_{c}$ and $\widehat{\Pa}(X^{j}_t) \ \forall j, t$}.
\end{algorithmic} 
\end{algorithm}

\subsection{Theoretical Guarantees}
Theorem \ref{main: theorem 4.1} establishes that if the window $W_i$ only contains samples from a single causal mechanism, meaning there is no change point included in this window, the estimated relative Pearson Divergence $\widehat{PE}_{\alpha\beta_i}$ is close to zero with high probability. Theorem \ref{main: theorem 4.2} provides a confidence interval for the change point estimator $\widehat{T}^j_c$. Proofs can be found in Appendix~\ref{app:Theoretical_Guarantees}.

\begin{theorem}\label{main: theorem 4.1}
Let $\{\widehat{PE}_{\alpha\beta_i}\}_i$ be the estimated PE series for one time series segment $X^j(\Lambda)\subsetneq \Xb^j \subsetneq V$ and $T^j_c$ denote the true change point in this time series segments. Under certain assumptions, we have that $\forall i \in \{i: i\ns+2\nw-1<T^j_{c}\}$
\begin{align*}
    \Pr \bigl(\max \{\widehat{PE}_{\alpha\beta_i}\}_i &< o(1)\bigr)>1-\frac{\aw-2}{\bs\log \tsub }-\frac{\aw}{\tsub},
\end{align*}
where $\bs=\lfloor \frac{\log \tsub}{\ns}\rfloor$,  $\aw=\lceil\frac{\tsub}{\nw}\rceil$ and $\tsub \coloneqq |X^j(\Lambda)|$.
\end{theorem}
The window index $i$ satisfying $i\ns+2\nw-1<T^j_{c}$ guarantees that all the samples in $W_i$ are collected from the same distribution. Theorem \ref{main: theorem 4.1} states that the maximum estimated PE divergence series obtained from such windows are bounded by any positive constant with probability $1-\frac{\aw-2}{\bs\log \tsub }-\frac{\aw}{\tsub}$ given large enough $\nw$. Note that $a_\text{w}$ and $b_\text{st}$ are constants whose specific values are determined by the chosen window $2\nw$ and stride size $\ns$.
\begin{theorem} \label{main: theorem 4.2}
Let $\{\widehat{PE}_{\alpha\beta_i}\}_i$ be the estimated PE series for one time series segment $X^j(\Lambda)\subsetneq \Xb^j \subsetneq V$ and $T^j_{c}$ denote the true change point in this time series segments. $\widehat{T}_{c}^j$ denotes the estimator of $T^j_{c}$ obtained by:
\begin{align}
    \widehat{T}_{c}^j=\arg_i\max\{\widehat{PE}_{\alpha\beta_i}\}_i
\end{align}
Under certain assumptions, we have that given large enough $\nw$, $\forall i \in [\tau_\text{max}+1,T]$
\begin{align}
     &\Pr\bigl(|\widehat{T}_{c}^j-T^j_c| < 2\nw\bigr)>\notag\\
     &\;\;\;\;\;\;\;\;\;\;\;\;\;\;\;\left(1-\frac{\aw-2}{\bs\log \tsub }-\frac{\aw}{\tsub}\right)\left(1-\frac{1}{\nw}\right).
\end{align}
where $\bs=\lfloor \frac{\log \tsub}{\ns}\rfloor$, $\aw=\lceil\frac{\tsub}{\nw}\rceil$, and $\tsub \coloneqq |X^j(\Lambda)|$.
\end{theorem}
The change point estimator has an estimation error smaller than the total window size with a known probability.

\subsection{Computational Complexity Analysis}
There is a trade-off between computational complexity and result accuracy in the proposed algorithm, with a preference for the latter when significant causal relationships exist. The proposed algorithm functions in two distinct phases as follows:

Phase One (causal discovery) focuses on uncovering the underlying causal structure of the entire non-stationary time series.
Phase Two (change point detection) centers on change point detection within each time series segment, leveraging the causal structures estimated in Phase One.

In Phase One, the worst-case computational
complexity in PCMCI \citep{runge2019detecting} is given by $n^3\tau_\text{max}^2+n^2\tau_\text{max}$ total conditional independence tests. The running time of PCMCI is then scaled by the time series length $T$ and the size of the conditioning set in conditional independence tests.

In Phase Two, the leading term of the asymptotic convergence rates of $\widehat{PE}$ is $n_\text{w}^{-1/2}$, as discussed in \citep{yamada2013relative}. In our algorithm, each time series segment contains a total of $\left\lfloor \frac{\tsub - 2\nw}{\ns} \right\rfloor + 1$ sliding windows, which implies that the number of $\widehat{PE}$ estimates is also $\left\lfloor \frac{\tsub - 2\nw}{\ns} \right\rfloor + 1$. Consequently, the total number of estimators $\{\widehat{PE}^{j,\Lambda}_{\alpha\beta_i}\}_{i,\Lambda}$ for $\Xb^j$ with $j\in[n]$ can be expressed as:

\begin{align}
    |\{\widehat{PE}^{j,\Lambda}_{\alpha\beta_i}\}_{i,\Lambda}| \approx (\left\lfloor \frac{\frac{T}{|s^{|\widehat{S\Pa}(X^j_t)|}|} - 2\nw}{\ns} \right\rfloor + 1)|s^{|\widehat{S\Pa}(X^j_t)|}|
\end{align}
The RuLSIF method for estimating relative PE divergence is computationally efficient as the optimization process utilizes a kernel-based approach with a square loss.

The complexity of Phase One increases cubically with increasing dimensions of $n$. Compared to existing high-dimensional methods in \citep{killick2012optimal} and \citep{kovacs2023seeded}, the computational complexity of Phase Two depends on the sparsity of the underlying causal structure among the $n$-variate time series, rather than directly on $n$. If the causal structure is sparse, the complexity of Phase Two remains unaffected. Specifically, for a target time series $\Xb^j$, adding one more time series $\Xb^i$ does not increase the complexity of detecting change points in $\Xb^j$, assuming there are no time-lagged causal effects from $\Xb^i$ to $\Xb^j$. This advantage distinguishes our approach from existing methods.

However, a limitation arises when the parent size of each variable increases linearly with $n$, as the number of time series segments for $\Xb^j$ in Phase Two will increase exponentially with the size of the superset parent set $\widehat{\SPA}(X_t^j)$.

As shown in Fig. \ref{fig 6} (Runtime for algorithms), the additional computational time required by Causal-RuLSIF, compared to RuLSIF, arises from Phase One and the need to estimate the PE series across multiple univariate time series segments (Phase Two), rather than processing the entire time series like other baselines. This introduces a trade-off between analyzing multiple univariate time series segments and handling one high-dimensional data. The computational burden is balanced with the advantages of segment-wise analysis.

The complexity of Phase One scales with $T$, depending on the specific choice of the conditional independence (CI) test. As the proposed algorithm applied sliding window techniques, the complexity of Phase Two increases linearly with $T$, consistent with other methods in \citep{killick2012optimal} and \citep{kovacs2023seeded}.

\section{EXPERIMENTS}\label{sec:experiments}
In this section, we present an empirical evaluation of our approach compared to existing methods, using both synthetic and real-world datasets. Section~\ref{simulations binary} analyzes simulation results on binary multivariate time series, while Section~\ref{case study} offers a case study. The Python code is provided at \href{https://github.com/CausalML-Lab/CausalCPD}{https://github.com/CausalML-Lab/CausalCPD}.

Sample efficiency presents a significant challenge for our algorithm, given that \( \tsub \approx T/|s^{|\widehat{S\Pa}(X^j_t)|}|\). To address this issue, we either implement the k-PC algorithm from \citep{kocaoglu2024characterization} or directly constrain the candidate parent set using the top-K causal strengths derived from the PCMCI algorithm. These enhancements enable our algorithm to effectively handle non-binary discrete-valued multivariate time series with larger domain sizes (\( D > 2 \)) and a greater number of component time series (\( n > 3 \)). Comprehensive experimental results are provided in Appendix~\ref{app:k-PC and Top-K}.

Additionally, we extend our approach to detect multiple change points with prior knowledge of their quantity (Assumption A8). The experimental results and corresponding theoretical guarantees are also available in Appendix~\ref{app:multiple change points}.

\subsection{Simulations on binary multivariate time series} \label{simulations binary}

 In this section, we have five baseline algorithms, including RuLSIF algorithm in \citep{liu2013change}, traditional method change in mean (CIM) in \citep{vostrikova1981detecting}, changeforest algorithm (RF) in \citep{londschien2023random}, ecp algorithm in \citep{james2013ecp} and kscp3o algorithm in \citep{zhang2017pruning}. 
 All experimental code will be available online.

The details of synthetic binary time series generation can be found in Appendix \ref{app:data generation}.

We have two methods to evaluate the performance of the algorithms. The first method is to calculate the estimation error using $\frac{|\widetilde{T}^j_c - T^j_c|}{T}$. The second method is to construct an ROC curve by setting an interval length, denoted as $2Q$. With the change point estimator $\widetilde{T_c}$ and interval length $Q$, we increment a counter by 1 if the true change point $T^j_c$ falls within the interval $[\widetilde{T}^j_c - Q, \widetilde{T}^j_c + Q]$. This count is then averaged over the total number of univariate time series in the 100 random trials. This is a common metric for measuring the performance of the change point detection algorithm, as described in \citep{liu2013change} and \citep{harchaoui2008kernel}.

\begin{figure*}[t!]
\centering     
\subfigure[Influence of $|\SPA(X^j_t)|$ on estimation error]{\label{fig2:a}\includegraphics[height=52mm,width=65mm]{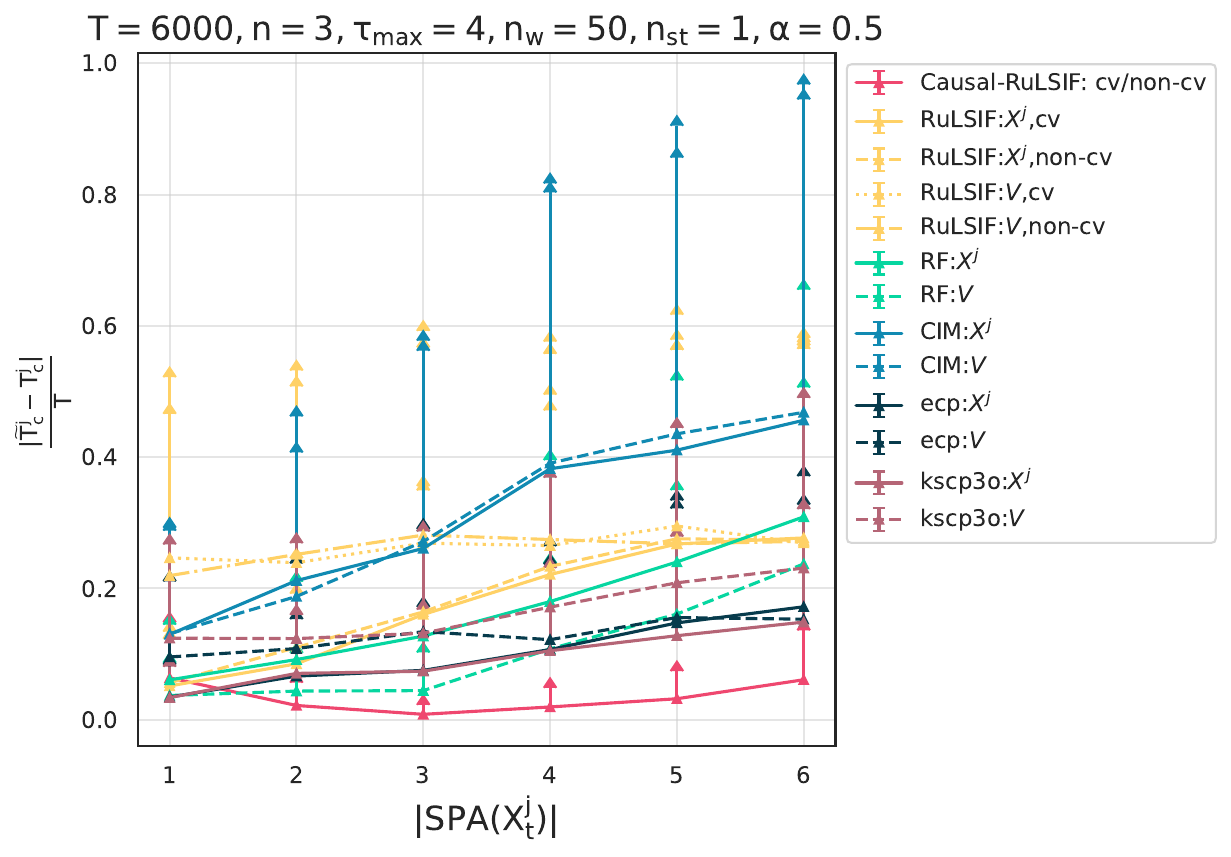}}
\subfigure[ROC curves for different $|\SPA(X^j_t)|$]{\label{fig2:b}\includegraphics[height=52mm,width=73mm]{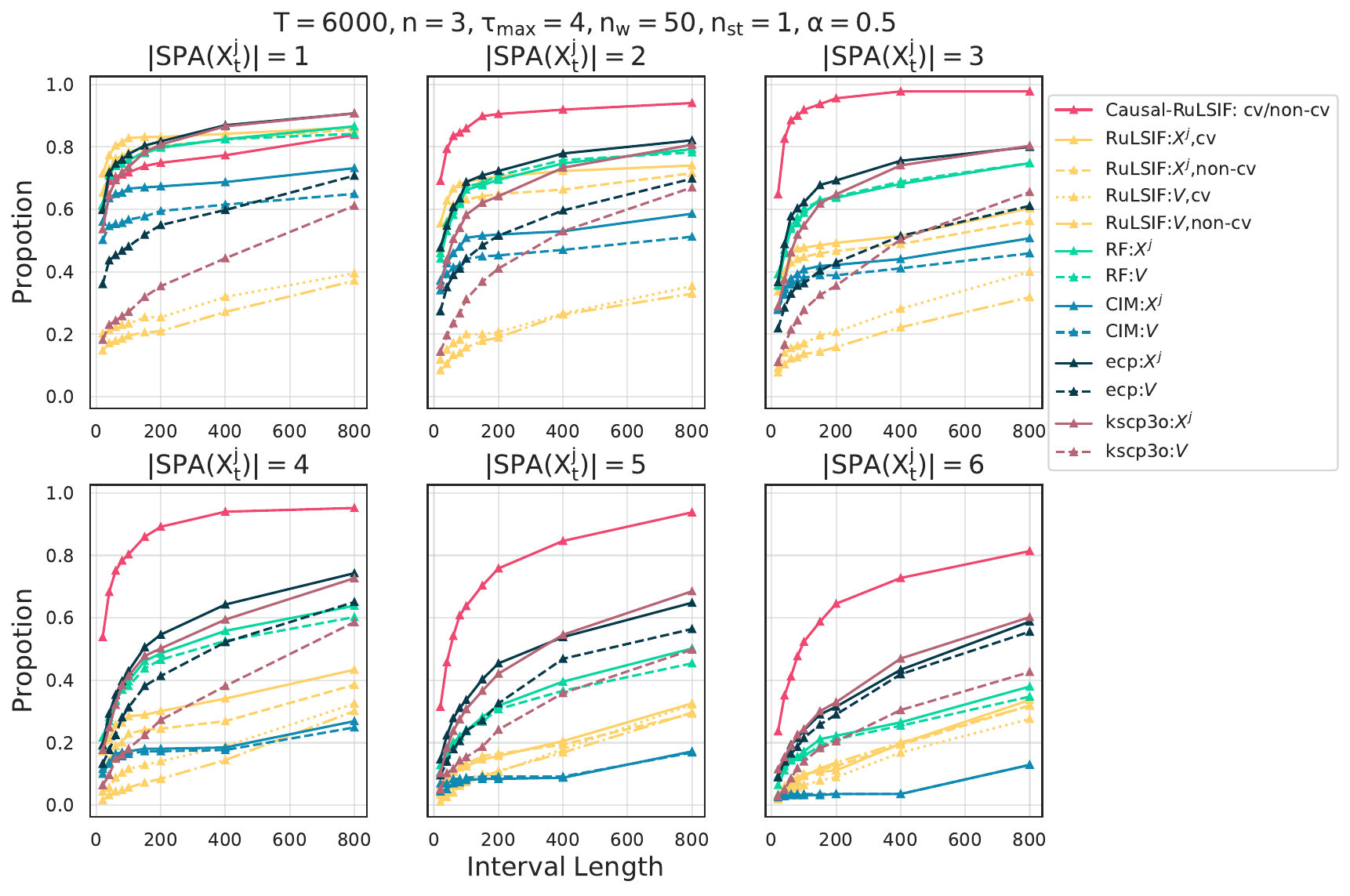}}
\caption{ Causal-RuLSIF is tested on 3-multivariate time series with $T=6000,\tau_{\max}=4,\nw=50,\ns=1$ with \emph{soft mechanism change}. Every line with a different color corresponds to a different algorithm and different linestyle corresponds to a different setting. $X^j$ in the legend means the algorithm is applied to each univariate time series while $V$ means the algorithm is used for the whole $n$-variate time series $V$. Every marker corresponds to the average error or average accuracy rate over 100 random trials. The error bar represents the standard error for the averaged statistics. a) Influence of $|\SPA(X^j_t)|$ on estimation error $\frac{|\widetilde{T}^j_c-T^j_c|}{T}$. b) ROC curves for different $|\SPA(X^j_t)|$.}
\end{figure*}

Please note that in the RuLSIF method, the kernel width $\sigma$ and the regularization parameter in the kernel function are typically chosen using cross-validation, as outlined in \citep{liu2013change} and \citep{harchaoui2008kernel}. This approach is justified, as high-dimensional data can render these parameters more sensitive. However, for binary time series, cross-validation may not be necessary when applying our method, Causal-RuLSIF. One possible reason is that our algorithm avoids the complexities associated with high-dimensional data $V$ by focusing on the analysis of one-dimensional time series segments $X^j(\Lambda)\subsetneq V$. As shown in Fig.\ref{fig2:a} and Fig.\ref{fig2:b}, the red line represents our algorithm. With or without the cross-validation technique does not influence its performance. For RuLSIF, the cross-validation (cv) and no cross-validation (non-cv) do not overlap.

For all the baselines, we either apply them to the entire $n$-variate time series or to each component $\Xb^j$. In the former scenario, multiple change point estimations are expected, making it difficult to determine which corresponds to which univariate time series. For the RuLSIF method, we select the estimations with the top $n$ change scores and randomly assign them to each univariate time series. Regarding RF, CIM, ecp, and kscp3o, even if the optimal estimations are selected when $T^j_c$ is known, their performance does not surpass ours. In the case of RF and CIM, the change point detection relies on the significance of certain statistics. If these methods fail to detect any change point, we set the estimated change point to $T$.

From Fig.\ref{fig2:a} and Fig.\ref{fig2:b}, Causal-RuLSIF is not optimal when $|\SPA(X^j_t)|=1$. In this scenario, each variable $X^j_t$ only receives an incoming edge from its only parent $X^j_{t-1}$, indicating no correlation among different time series. This special structure makes certain baselines more advantageous, particularly considering the potential false positive edges identified by the proposed algorithm, which lacks prior knowledge of the absence of correlations and limited sample size. Therefore, it is reasonable for some baselines to exhibit better performance when applied individually to each time series. 

However, when $|\SPA(X^j_t)|>1$, our algorithm outperforms others, as it focuses on shifts in causal mechanisms given other time series. The performance of Causal-RuLSIF decreases as $|\SPA(X^j_t)|$ increases because the effective sample size decreases with the number of parent configurations, which is $2^{|\SPA(X^j_t)|}$ for binary time series.

Additional experiments on \emph{hard mechanism change}, the impact of $\nw$, the relative location of $T^1_c$ and $T^2_c$, and runtime analysis are provided in Appendix \ref{app:simulations}.

From Fig.\ref{fig4:a} and Fig.\ref{fig4:b}, it is evident that increasing the length $T$ of the time series $V$ enhances the performance of Causal-RuLSIF, thus practically validating the consistency of the algorithm.
\begin{figure*}[t!]
\centering     
\subfigure[Influence of $T$ on estimation error]{\label{fig4:a}\includegraphics[height=52mm,width=65mm]{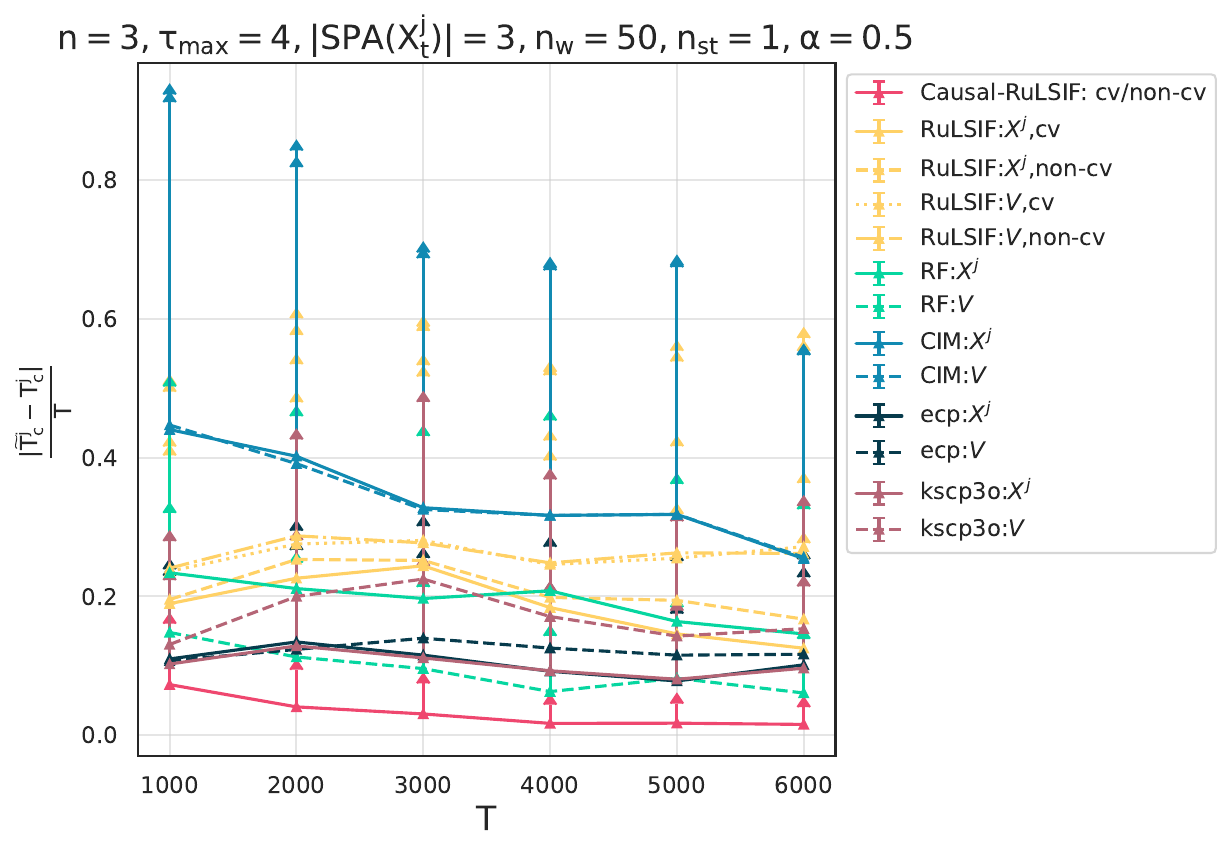}}
\subfigure[ROC curves for different $T$]{\label{fig4:b}\includegraphics[height=52mm,width=73mm]{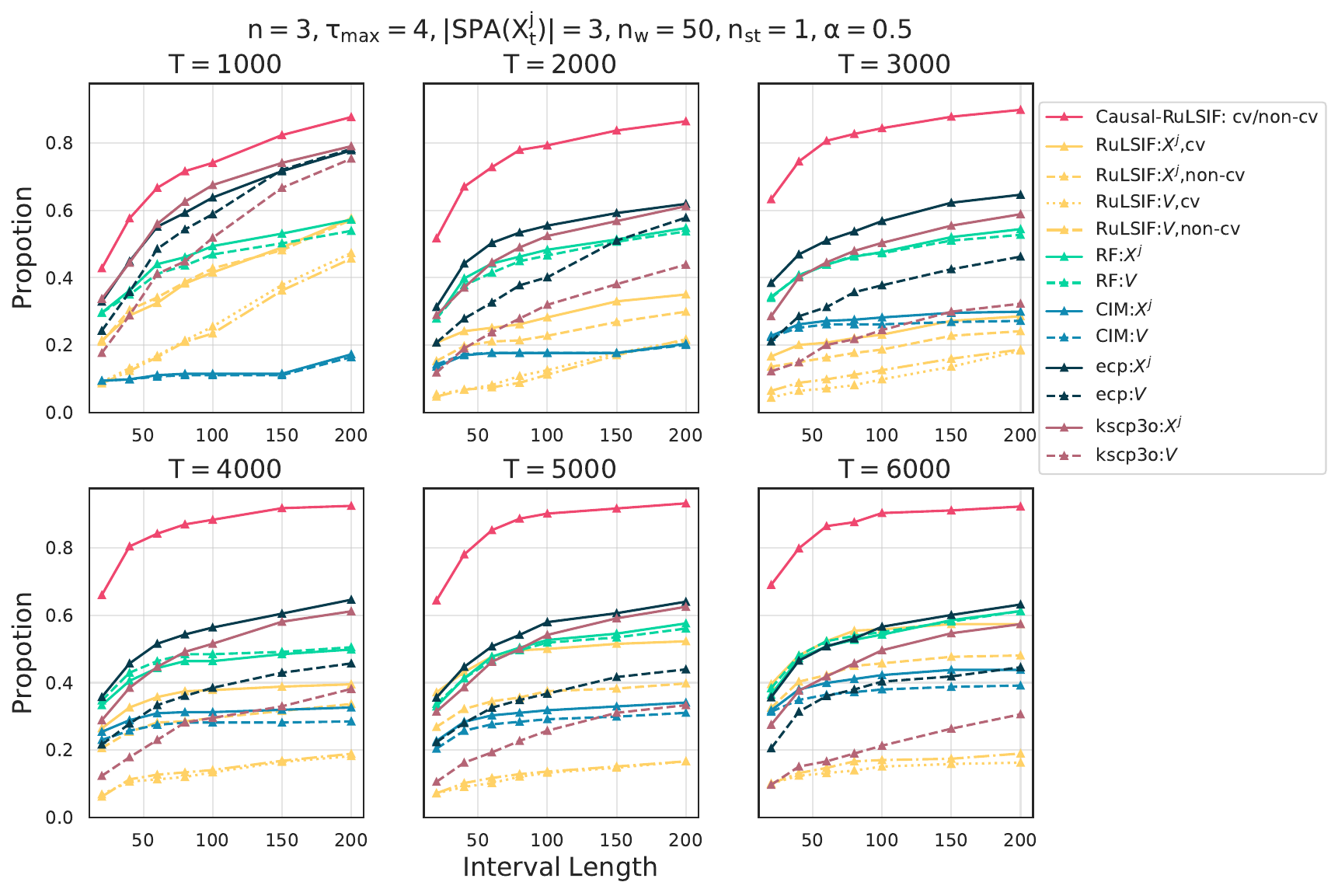}}
\caption{Causal-RuLSIF is tested on 3-multivariate time series with $|\SPA(X^j_t)|=3,\tau_{max}=4,\nw=50,\ns=1$ with \emph{soft mechanism change}. Every line with a different color corresponds to a different algorithm and different linestyle corresponds to a different setting. $X^j$ in the legend means the algorithm is applied to each univariate time series while $V$ means the algorithm is used for the whole $n$-variate time series $V$. Every marker corresponds to the average error or average accuracy rate over 100 random trials. The error bar represents the standard error for the averaged statistics. a) Influence of $T$ on estimation error $\frac{|\widetilde{T}^j_c-T^j_c|}{T}$. b) ROC curves for different $T$.}
\end{figure*}

\subsection{Case Study}\label{case study}
Here, we construct an experiment with a real-world air pollution dataset. This dataset monitors the amount of PM$_\text{10}$ (coarse particles with a diameter between 2.5 and 10 micrometers) in the air. The 3-variate time series data records the hourly concentration of PM$_\text{10}$ across three counties in California—Fresno, Mono and Monterey—from Jan 
 to June 2023. There are a total of 4305 samples. Let $\Xb^\text{Fr}$, $\Xb^\text{Mono}$ and $\Xb^\text{Mont}$ denote the indicators of PM${10}$ exceeding 10 across the three counties.

Using Causal-RuLSIF, the change point estimators $\widetilde{T}^j_c$ for each $\Xb^j$ where $j \in [3]$ are shown in the Table~\ref{table:1}, along with the parent sets before and after the estimated change point. 

Based on the results, assuming there is one change point in PM${10}$ concentration in Fresno, Mono, and Monterey during the first half of 2023, the causal mechanism of PM$_\text{10}$ in Fresno is likely to shift on May 8, 2023. Additionally, while the PM$_{10}$ levels in Fresno and Monterey are not influenced by other counties, a causal link from Monterey to Mono has emerged after February 2, 2023.

\begin{table}[t]
\vspace{2mm}
\centering
\caption{Causal-RuLSIF in PM$_\text{10}$ dataset}
\label {table:1}  
\begin{tabular}{ c  c  c }
\toprule
$X$ & $\widetilde{T}^j_c$ & {$\widehat{\Pa}(X^j_{t<\widetilde{T}^j_c})$; $\widehat{\Pa}(X^j_{t\geq \widetilde{T}^j_c})$}\\
\midrule
 $\Xb^\text{Fr}$ & 05/08/23 &  $\{X^\text{Fr}_{t-1}\}$; $\{X^\text{Fr}_{t-1,t-2,t-3}\}$ \\


 $\Xb^\text{Mono}$ & 02/01/23 & $\{X^\text{Mono}_{t-1}\}$; $\{X^\text{Mono}_{t-1,t-2},X^\text{Mont}_{t-3}\}$ \\
 
 $\Xb^\text{Mont}$ & 04/04/23 &$\{X^\text{Mont}_{t-1}\}$;$\{X^\text{Mont}_{t-1}\}$\\
 \bottomrule
\end{tabular}
\end{table}
Without the ground truth and relevant knowledge about air pollution and other climate-related information, it is difficult to determine the significance of the case study results. We hope this real data application can offer insights for experts in other fields on detecting change points in causal mechanisms in practice. Additional case studies can be found in Appendix~\ref{app:case study}.
 


\section{CONCLUSION}\label{sec:conclusion}
In this paper, we introduced a novel change point detection algorithm, Causal-RuLSIF, to identify significant changes in causal mechanisms for discrete-valued time series data. By integrating a post-processing causal discovery stage with a novel dynamic divergence estimation, our algorithm accurately detects when causal mechanism shifts occur without imposing constraints on the form of the shift. We provide a theoretical uncertainty analysis of the change point estimator. Our empirical evaluation demonstrates the consistency and robustness of the proposed algorithm. The limitations of the algorithm are discussed in Appendix~\ref{app:limitations}.
\section{ACKNOWLEDGEMENTS}
This research has been supported in part by NSF CAREER 2239375, IIS 2348717, Amazon Research Award and Adobe Research. We sincerely thank the anonymous reviewers for their insightful and constructive feedback, which greatly improved the quality of this manuscript.
\clearpage

\clearpage
\section*{CHECKLIST}

 \begin{enumerate}

 \item For all models and algorithms presented, check if you include:
 \begin{enumerate}
   \item A clear description of the mathematical setting, assumptions, algorithm, and/or model. [\textcolor{red}{Yes}/No/Not Applicable]
   \item An analysis of the properties and complexity (time, space, sample size) of any algorithm. [\textcolor{red}{Yes}/No/Not Applicable]
   \item (Optional) Anonymized source code, with specification of all dependencies, including external libraries. [\textcolor{red}{Yes}/No/Not Applicable]
 \end{enumerate}

 \item For any theoretical claim, check if you include:
 \begin{enumerate}
   \item Statements of the full set of assumptions of all theoretical results. [\textcolor{red}{Yes}/No/Not Applicable]
   \item Complete proofs of all theoretical results. [\textcolor{red}{Yes}/No/Not Applicable]
   \item Clear explanations of any assumptions. [\textcolor{red}{Yes}/No/Not Applicable]     
 \end{enumerate}

 \item For all figures and tables that present empirical results, check if you include:
 \begin{enumerate}
   \item The code, data, and instructions needed to reproduce the main experimental results (either in the supplemental material or as a URL). [\textcolor{red}{Yes}/No/Not Applicable]
   \item All the training details (e.g., data splits, hyperparameters, how they were chosen). [\textcolor{red}{Yes}/No/Not Applicable]
         \item A clear definition of the specific measure or statistics and error bars (e.g., with respect to the random seed after running experiments multiple times). [\textcolor{red}{Yes}/No/Not Applicable]
         \item A description of the computing infrastructure used. (e.g., type of GPUs, internal cluster, or cloud provider). [\textcolor{red}{Yes}/No/Not Applicable]
 \end{enumerate}

 \item If you are using existing assets (e.g., code, data, models) or curating/releasing new assets, check if you include:
 \begin{enumerate}
   \item Citations of the creator If your work uses existing assets. [\textcolor{red}{Yes}/No/Not Applicable]
   \item The license information of the assets, if applicable. [\textcolor{red}{Yes}/No/Not Applicable]
   \item New assets either in the supplemental material or as a URL, if applicable. [Yes/No/\textcolor{red}{Not Applicable}]
   \item Information about consent from data providers/curators. [\textcolor{red}{Yes}/No/Not Applicable]
   \item Discussion of sensible content if applicable, e.g., personally identifiable information or offensive content. [Yes/No/\textcolor{red}{Not Applicable}]
 \end{enumerate}

 \item If you used crowdsourcing or conducted research with human subjects, check if you include:
 \begin{enumerate}
   \item The full text of instructions given to participants and screenshots. [Yes/No/\textcolor{red}{Not Applicable}]
   \item Descriptions of potential participant risks, with links to Institutional Review Board (IRB) approvals if applicable. [Yes/No/\textcolor{red}{Not Applicable}]
   \item The estimated hourly wage paid to participants and the total amount spent on participant compensation. [Yes/No/\textcolor{red}{Not Applicable}]
 \end{enumerate}

 \end{enumerate}

\newpage
\appendix




\onecolumn
\aistatsappendixtitle{Causal Discovery-Driven Change Point Detection in Time Series: \\
Supplementary Materials}
\section*{Appendix Outline}\label{app:outline}
In Section~\ref{app:prelims}, Section~\ref{app:tss} provides a toy example of time series segments. Section~\ref{app:RuLSIF} introduces the proposed dynamic relative PE divergence and explains our rationale for modifying relative PE divergence in our algorithm. In Section~\ref{app:assumptions}, assumptions are stated. Section~\ref{app:Theoretical_Guarantees} contains detailed proofs of theorems. Section~\ref{app:related work} discusses related work on causal discovery methods.

Details of the simulated time series generation process can be found in Section~\ref{app:data generation}. Additional experimental results, continuing from the experiment section in the main paper, are presented in Section~\ref{app:simulations}. Section~\ref{app:k-PC and Top-K} addresses sample efficiency issues in practice, and Section~\ref{app:multiple change points} extends our algorithm to handle multiple change point cases. A cautionary case study demonstrating the necessary requirements of our algorithm is provided in Section~\ref{app:case study}. Finally, the limitations are discussed in Section~\ref{app:limitations}.

\section{Preliminaries}\label{app:prelims}
\subsection{Time Series Segments}\label{app:tss}
In Fig.~\ref{full causal graph}a), for $\Xb^1$, $\SPA(X^1_t)=\{X^1_{t-1},X^1_{t-3},X^2_{t-2},X^3_{t-1}\}$. Assuming $V$ is binary time series with $D=[0,1]$, we have:
\begin{align}
   \Sigma^1 \coloneqq  \begin{blockarray}{cccccc}
  & \matindex{$X^1_{t-1}$}&\matindex{$X^1_{t-3}$} & \matindex{$X^2_{t-2}$} & \matindex{$X^3_{t-1}$}& \\
    \begin{block}{c[cccc]c}
      \matindex{1} & 0 & 0 & 0 & 0 \\
      \matindex{2} & 1 & 0 & 0 & 0   \\
       \matindex{3} & 0 & 1 & 0 & 0   \\
      \matindex{$\cdots$} &\text{$\cdots$} & \text{$\cdots$} &\text{$\cdots$}& \text{$\cdots$}&  \\
      \matindex{$s^{|\SPA(X^j_t)|}$} & 1 & 1 & 1 & 1 & {\text{$s^{|\SPA(X^j_t)|} \times |\SPA(X^j_t)|$}}\\
    \end{block}
  \end{blockarray}
\end{align}
With $s=|D|=2$ and $|\SPA(X^j_t)|=4$, there are $16=2^4$ configurations of $\SPA(X^j_t)$. Hence there are $16$ time series segments of $\Xb^1$. More specifically, $X^1(1)\coloneqq \{X^j_t: t\in [\tau_\text{max}+1,T],\pa(X^j_t)=[0,0,0,0]\}$, $X^1(2)\coloneqq \{X^j_t: t\in [\tau_\text{max}+1,T],\pa(X^j_t)=[0,1,0,0]\}$ and so on. For $\Xb^3$, $\SPA(X^3_t)=\{X^1_{t-1},X^2_{t-1},X^3_{t-1}\}$ and hence there are total $2^3=8$ configurations. Fig.\ref{full causal graph}b) shows the construction of $X^3(1)$ and Fig.\ref{full causal graph}c) illustrates $8$ total time series segments of $\Xb^3$.

\subsection{RuLSIF: Robust Distribution Comparison}\label{app:RuLSIF}

Given two distributions $p(\cdot)$ and $p'(\cdot)$ defined over the same support, there exist many metrics measuring the distance between them. In~\citep{yamada2013relative}, proposed a method named RuLSIF, with an $\alpha$-relative divergence estimation $r_\alpha(x)$ and a corresponding metric, $\alpha$-relative Pearson Divergence $PE_\alpha$ defined as following:
\begin{align}
    &r_\alpha(x)\coloneqq \frac{p(x)}{(1-\alpha)p(x)+\alpha p'(x)}\coloneqq\frac{p(x)}{q_{\alpha}(x)}\label{r_alpha}\\
    &PE_{\alpha}\coloneqq \frac{1}{2}\mathbb{E}_{x \sim q_\alpha}[(r_\alpha(x)-1)^2]
\end{align}
where $\alpha$ is a parameter used to bound the value of $r_\alpha(x)$.
With IID samples, we can obtain a direct density-ratio estimator $\hat{r}_\alpha(x)$ using a kernel function, by minimizing a squared loss function. The estimator has been proven to have a non-parametric convergence speed.
In \citep{liu2013change}, the method was applied to tackle the change point detection problem in time series using sliding windows. By dividing the time series into retrospective segments, a sequence of PE divergence estimated by the RuLSIF method is obtained by assessing the distribution divergence between samples from two consecutive segments. Peaks in the divergence score $PE_{\alpha}$ can indicate change points within the joint distribution. The retrospective segments in \citep{liu2013change} are high-dimensional, even for univariate time series $V$.


In our proposed method, we also utilize sliding windows. However, since the focus shifts from the joint distribution to the causal mechanism, we do not construct high-dimensional retrospective segments from $V$. Instead, we create one-dimensional consecutive segments on the time series segments for each $\Xb^j\subsetneq V$. This approach avoids the Curse of dimensionality and enhances the method's robustness to the hyperparameters in the kernel functions, as will be verified in the experiment section. 

Since we construct time series segments in our framework, the samples are IID. The distributions that need to be compared are formalized as follows:
\begin{align}
 p(X^j_{t_1}|\spa(X^j_{t_1})=\Sigma_{\Lambda}) \text{ vs } p(X^j_{t_2}|\spa(X^j_{t_2})=\Sigma_{\Lambda})
\end{align}
where $\Lambda \in [s^{|\SPA(X^j_t)|}]$, $t_1<T_{c}$ and $t_2\geq T_{c}$.

To simplify matters, we use $p(x)$ to denote $p(X^j_{t_1}|\spa(X^j_{t_1})=\Sigma_{\Lambda})$ and $p'(x)$ represent $p(X^j_{t_2}|\spa(X^j_{t_2})=\Sigma_{\Lambda})$ in the rest of the paper.


  \begin{figure}[t!]
    \centering     
     \subfigure[The first time series segment $X^3(1)$ from Fig.\ref{full causal graph}b) for $\Xb^3$]{\label{fig:sliding window}\includegraphics[width=149mm]{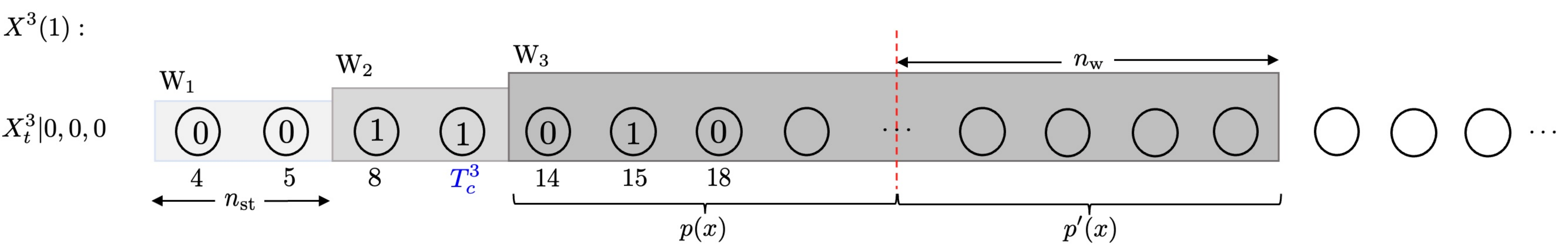}}
     
    \subfigure[$\alpha$-relative Pearson (PE) divergence series for an one arbitrary $3$-variate time series with domain size 2]{\label{fig:example 1 pe divergence}\includegraphics[width=149mm]{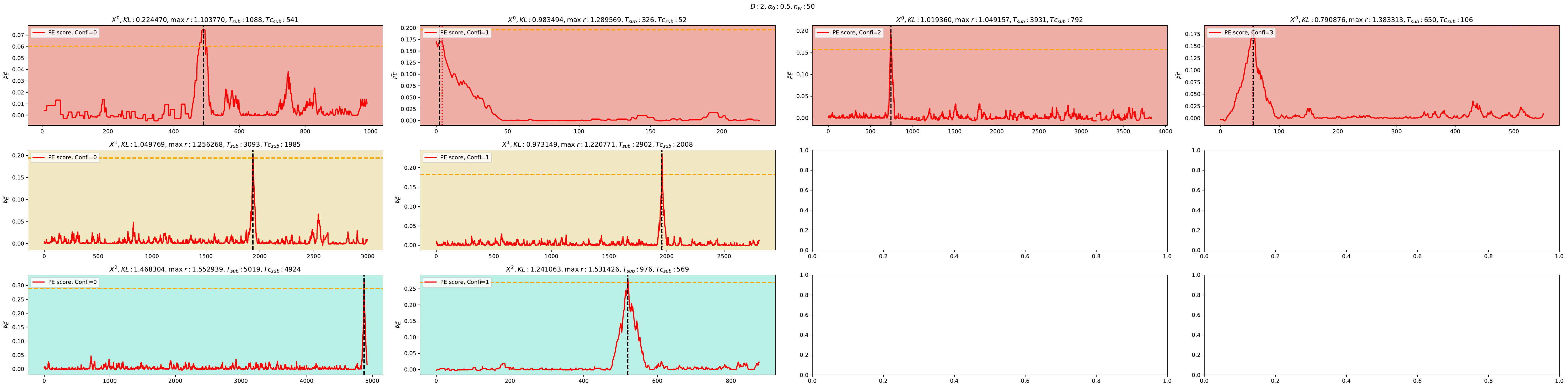}}
    \subfigure[$\alpha$-relative Pearson (PE) divergence series for an one arbitrary $3$-variate time series with domain size 3]{\label{fig:example 2 pe divergence}\includegraphics[width=149mm]{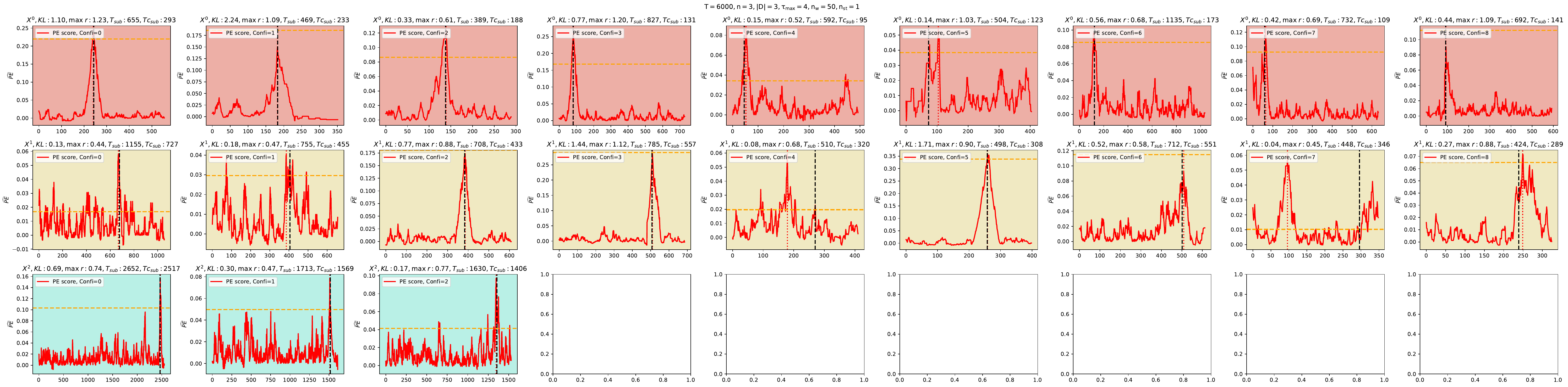}}
            \caption{a) The first time series segment $X^3(1)$ from Fig.\ref{full causal graph}b) for $\Xb^3$. This toy example illustrates the sliding windows with $\nw$, $\ns$ for $\{X^3(\Lambda)\}_{\Lambda\in[8]}$. b) and c) illustrate the PE divergence series under different parent configurations for an arbitrary 3-variate time series. The y-axis represents the PE score, and the x-axis represents the index $i$ of the shifting window. The plots are organized into three rows, each corresponding to a univariate time series $X^j$, where $j \in [3]$. Each column represents a particular parent configuration. For instance, there are four configurations for $\textbf{X}^1$, while $\textbf{X}^2$ and $\textbf{X}^3$ each have only two configurations in (b). The black vertical line indicates the true $\widehat{T}^j_c$ within that time series segment, and the red vertical line indicates the location of the maximum estimated PE score through Causal RuLSIF. The yellow horizontal line represents the true PE score.
            }
    \end{figure}

Fig.\ref{fig:sliding window} provides a toy example illustrating how the sliding window operates on time series segments. A single PE divergence score is generated using two sets of samples, one from each half of the window. As the sliding window shifts from the start to the end of a time series segment, a PE divergence series is obtained. This series is represented in each subplot in Fig.~\ref{fig:example 1 pe divergence} and Fig.~\ref{fig:example 2 pe divergence}. Let $W_i$ denote the $i$th window, where $W^1_i$ represents the first half and $W^2_i$ represents the second half.

In the $i$th sliding window containing one change point, without loss of generality, assume that the change point is within the second half, $W^2_i$, of the window. Thus, while the samples in $W^1_i$ come from $p(x)$, the samples in $W^2_i$ come from a mixture distribution $(1-\beta_i)p(x) + \beta_i p'(x)$, where $\beta_i$ is the proportion of samples from $p'(x)$ within $W^2_i$. Note that $\beta_i$ is an unknown parameter, while $i$ is the known window index. Therefore, it is more accurate to denote the divergence score for the $i$th window as $PE_{\alpha\beta_i}$. To clarify, we will use $PE_{\alpha\beta_i}$ instead. By replacing $p'(x)$ with $(1-\beta_i)p(x) + \beta_i p'(x)$ in Eq.\ref{r_alpha}, we obtain the $\alpha\beta_i$-relative divergence estimation $r_{\alpha\beta_i}(x)$ and the $\alpha\beta_i$-relative Pearson Divergence for the $i$th window:
\begin{align}
  &r_{\alpha\beta_i}(x)\coloneqq \frac{p(x)}{(1-\alpha\beta_i)p(x)+\alpha\beta_i p'(x)}\coloneqq\frac{p(x)}{q_{\alpha\beta_i}(x)}\label{r_alphai}\\ 
  &PE_{\alpha\beta_i}\coloneqq \frac{1}{2}\mathbb{E}_{x \sim q_{\alpha\beta_i}}[(r_{\alpha\beta_i}(x)-1)^2]\label{eq.24}
\end{align}
and
\begin{align}
    \beta_i := (1-\frac{T_c-(in_\text{st}+n_\text{w})}{n_\text{w}})\mathbf{1}({T_c-in_\text{st}\geq n_\text{w} \text{ and } T_c-in_\text{st}<2n_\text{w} })\label{beta}
\end{align}
where $T_\text{sub}$ is the target time series segment, $n_\text{w}$ is the half window size and $n_\text{st}$ is the stride size.

Fig.~\ref{fig:example 1 pe divergence} and Fig.~\ref{fig:example 2 pe divergence} display the PE divergence series for a 3-variate time series with domain sizes 2 and 3, respectively. These series were generated by applying a shifting window to each segment of the univariate time series $\Xb^j$ for $j\in[3]$. In each figure, each subplot represents the PE divergence series for a single time series segment, while the rows of subplots collectively show multiple PE divergence series for the corresponding univariate time series $\Xb^j$ for $j\in[3]$. Each PE divergence value is calculated from a window, forming the PE divergence series by sliding the window from $t=0$ to $t=T_{\text{sub}}-2n_{\text{w}}$.

The reasons we chose the $\alpha$-relative density ratio PE as the fundamental divergence measure in our framework are:
\begin{itemize}

    \item One motivation for choosing relative PE divergence is that its asymptotic properties have been explored in \citep{yamada2013relative}, allowing us to directly use these properties in our theorem. In \citep{yamada2013relative}, the authors theoretically demonstrate that the $\alpha$-relative PE divergence estimator, based on $\alpha$-relative density-ratio approximation, offers a more favorable non-parametric convergence speed than the standard density-ratio approach. Additionally, with a correctly specified parametric setup, the asymptotic variance of the proposed $\alpha$-relative PE divergence estimator remains independent of model complexity. This implies that the proposed estimator resists overfitting, even when applied to complex models.

    \item As a squared-loss variant of KL divergence, PE divergence estimator is computationally cheaper than KL divergence because it does not involve the log term. Compared to the unbounded density ratio $r(x) = \frac{p(x)}{p'(x)}$ in PE divergence, the $\alpha$-relative density ratio $r_{\alpha}(x) = \frac{p(x)}{\alpha p'(x) + (1 - \alpha) p(x)}$ in $\alpha$-relative PE divergence is always bounded above by $\frac{1}{\alpha}$. Relative PE divergence converges faster than PE divergence.

    \item The hyperparameter $\alpha$ in the $\alpha$-relative PE divergence can be modified as a dynamic index in the shifting window framework, representing the 'mixture' of two distributions. The denominator in $ r_{\alpha}(x) = \frac{p(x)}{\alpha p'(x) + (1-\alpha)p(x)}$ can be interpreted as a mixture distribution. A varying $\alpha_t$ can then act as a dynamic index for the "concentration" of $p^{\prime}(x)$ in this mixture as the shifting window moves from $t=0$ to $t=T$. This change allows for a transition from a static definition of $\alpha$-relative PE divergence to a dynamic conditional Relative PE divergence by incorporating $\beta_i$, as defined in Eq.~\ref{beta}, making it more suitable for time series data. 
\end{itemize}

\subsection{Assumptions}\label{app:assumptions}
\begin{enumerate}
\item[\bf A1.] \textbf{Sufficiency}: There are no unobserved confounders.
\item[\bf A2.] \textbf{Causal Markov Condition}: Each variable $X$ is independent of all its non-descendants, given its parents $\Pa(X)$ in $\mathcal{G}$.
\item[\bf A3.] \textbf{Faithfulness Condition \citep{pearl1980causality}}: Let $P$ be a probability distribution generated by $\mathcal{G}$. $\langle \mathcal{G}, P \rangle$ satisfies the Faithfulness Condition if and only if every conditional independence relation true in $P$ is entailed by the Causal Markov Condition applied to $\mathcal{G}$.
\item[\bf A4.] \textbf{No Contemporaneous Causal Effects}: Edges between variables at the same time are not allowed. 
\item[\bf A5.] \textbf{Temporal Priority}: Causal relations always point from the past to the future.
\item[\bf A6.]\label{assum:A7} 
\textbf{Boundary Separation Assumption}: There must be a minimum buffer period at both the beginning and the end of the time series where change points cannot be detected. More specifically, the change point for each time series $\Xb^j\in V$ cannot occur within the specified window size $\nw$.
\item[\bf A7.]\label{assum:A6} 
\textbf{One change point per component time series/if multiple change points exist, the temporal distance between any two must be at least $\Delta_c$}: $\forall j$, $c^j=1$, that is, each time series $\Xb^j\in V$ has only one change point/$\forall j$, if $c^j>1$, then $\forall c\in[c^j],T^j_{c}-T^j_{c-1}>\Delta_c$.
\item[\bf A8.]\label{assum:A8} \textbf{The number of change points is known.}
\item[\bf A9.]\label{assum:A9}
\textbf{Minimum Pearson Divergence}: For each $\Xb^j \subsetneq V$, there should exist a window $W^i$ satisfying $PE^2_{\alpha\beta_i}>\lVert r_{\alpha\beta_i}\rVert_{\infty}^2+\frac{(1-\alpha\beta_i)^2\lVert r_{\alpha\beta_i}\rVert_{\infty}^4}{4}+\frac{\alpha^2\beta_i^2\lVert r_{\alpha\beta_i}\rVert_{\infty}^4}{4}$ in at least one time series segments $X^j(\Lambda)$.
\end{enumerate}

Assumptions \textbf{A1-A5} are conventional and commonly employed in causal discovery methods for time series data. Our approach requires specific Assumptions \textbf{A6-A9} to be in place. To clarify, assumption \textbf{A6} is essential because our algorithm utilizes a series of sliding windows to obtain the divergence score by measuring the first half and the second half of the samples within each window. If the change point is too close to the beginning or end of the time series, the divergence score will not be significant enough to be detected. Assumption \textbf{A7} is required since the sliding windows are not directly applied on the original time series. Instead, multiple sub-time series are created and hence it is hard to impose the constraint on the minimum distance among multiple change points, such as in \citep{harchaoui2009regularized}, \citep{allen2018non} and \citep{chen2023graph}. Assumption \textbf{A8} provides a theoretical guarantee. 

Assumption \textbf{A9} is necessary for the proof, as it guarantees a significant difference between two distinct causal mechanisms. This assumption is crucial because detecting the change point successfully becomes highly unlikely if the two causal mechanisms are extremely similar. More specifically, the threshold in Assumption \textbf{A9} is set as a constant for any half window size $\nw$ for simplicity, making it a conservative choice. Assumption \textbf{A9} ensures that, among comparable conditional distributions from two causal mechanisms, at least one pair has a large divergence, such that the confidence interval of $\widehat{PE}$ divergence score strictly excludes 0 with a confidence level of $1 - \frac{1}{\nw}$.

The threshold can also be set as a variable that vanishes at a rate of $\nw^{-(\frac{1}{2} - q)}$ with $0<\epsilon<q<1/2$, while the confidence interval of $\widehat{PE}$ will shrink correspondingly with a lower confidence level of $1 - \frac{1}{\nw^{2q}}$.

In the experimental setting, no threshold was applied during the data generation process described in Appendix~\ref{app:data generation}.
\subsection{Theoretical Guarantees}\label{app:Theoretical_Guarantees}
Theorem \ref{pcmci} ensures that in the initial step of our algorithm, a superset of the true union parent set $\SPA(X^j_t)$ can be obtained for all $j\in [n]$. This guarantees the correct construction of time series segments $X^j(\Lambda)$, with IID samples. Theorem \ref{zero} establishes that if the window $W_i$ contains samples from a single conditional distribution, the estimated relative Pearson Divergence $\widehat{PE}_{\alpha\beta_i}$ is close to zero with high probability. The lemma \ref{PE close to PE hat} states that, as $\nw$ increases, the estimated relative Pearson Divergence will be close to the true relative Pearson Divergence up to some constant with high probability. 
The lemma \ref{increasing} shows that the relative Pearson Divergence will achieve maximum if all the samples in the first half window are from one distribution denoted by $p$ and the samples in the second half window are all from another distribution denoted by $p'$. Theorem \ref{Final} establishes a confidence interval for the change point estimator $\widehat{T}^j_c$. 

\begin{theorem}\label{pcmci}
Let $\SPA(X^j_t)$ denote the union parent set of $X^j_t$ and $\widehat{\SPA}(X^j_t)$ denote the estimated union parent set obtained from PCMCI algorithm on time series $\Xb^j$ with a Mechanism-Shift SCM, and the change point $T^j_c$ satisfies $T^j_c=\frac{T}{2}$. Under assumptions \textbf{A1-A5,A7} and with an oracle (infinite sample size limit), we have that:
\begin{align}
  \SPA(X^j_t)\subseteq \widehat{\SPA}(X^j_t)
\end{align}
\end{theorem}
Theorem \ref{pcmci} asserts that when samples are balanced, with half originating from one causal mechanism and the remaining half from another, the estimated union parent set encompasses the true union parent set.
\begin{proof}
The proof follows the same logic as the Lemma 3.2 and 3.3 in \citep{gao2023causal}. In the semi-stationary SCM, the samples are from multiple causal mechanisms and due to periodicity, heterogeneous samples from different causal mechanisms are balanced. However, in the mechanism-shift SCM, as per Assumption \textbf{A7}, there are only two causal mechanisms, and the samples are inherently unbalanced without specific clarification. With additional assumption $T^j_c=\frac{T}{2}$, we can draw the same conclusion as in \citep{gao2023causal}.
\end{proof}
Note that $\widehat{PE}^{j,\Lambda}_{\alpha\beta_i}$ series with parameter $\alpha$ and window index $i$ is a function of $\nw$, $\ns$, time series index $j$ and time series segments index $\Lambda$, shown as Eq.~\ref{r_alphai}-\ref{beta}. For simplicity, we use $\widehat{PE}_{\alpha\beta_i}$ instead of $\widehat{PE}^{j,\Lambda}_{\alpha\beta_i}(\nw,\ns)$ in the next section. Furthermore, we need to emphasize that $\tsub$ is not the length of $V$ but the length of the corresponding specific time series segment $X^j(\Lambda)$ . 
\begin{theorem}\label{zero}
Let $\{\widehat{PE}_{\alpha\beta_i}\}_i$ be the estimated PE series for one time series segments $X^j(\Lambda)\subsetneq \Xb^j \subsetneq V$ and $T^j_c$ denote the true change point in this time series segments. Under assumption \textbf{A6-A7}, we have that $\forall i \in \{i: i\ns+2\nw-1<T^j_{c}\}$
\begin{align}
    P\bigl(\max_i \{\widehat{PE}_{\alpha\beta_i}\}_i < o_p(1)\bigr)>1-\frac{\aw-2}{\bs\log \tsub }-\frac{\aw}{\tsub} 
\end{align}
where $\bs=\lfloor \frac{\log \tsub}{\ns}\rfloor$ and $\aw=\lceil\frac{\tsub}{\nw}\rceil$.
\end{theorem}
The window index $i$ satisfying $i\ns+2\nw-1<T_{c}$ guarantees that all the samples in $W_i$ are collected from the same distribution. Theorem \ref{zero} states that the maximum estimated PE score obtained from such windows are bounded by any positive constant with probability $1-\frac{\aw-2}{\bs\log \tsub }-\frac{\aw}{\tsub}$ if $\nw$ are larger than some threshold. In other words, $ \forall k>0, \exists N$ such that $\forall  \nw \geq N$:
\begin{align}
    P(\max_i \{\widehat{PE}_{\alpha\beta_i}\}_i< k)>1-\frac{\aw-2}{\bs\log \tsub }-\frac{\aw}{\tsub}
\end{align}

\begin{proof}
$\forall i \in \{i: i\ns+2\nw-1<T^j_c\}$,  the asymptotic expectation and variance of $\widehat{PE}_{\alpha\beta_i}$ is given by:
\begin{align}
\mathbb{E}(\widehat{PE}_{\alpha\beta_i})&=PE_{\alpha\beta_i}+o_p(\frac{1}{\sqrt{\nw}})\\
&=o_p(\frac{1}{\sqrt{\nw}})\\
\mathbb{V}(\widehat{PE}_{\alpha\beta_i})&=o_p(\frac{1}{\nw})\label{theorem}
\end{align}
The proof of the asymptotic expectation and variance can be found in Section B of \citep{yamada2013relative}.
By Chebyshev's inequality, we have:
\begin{align}
p\bigl(|\widehat{PE}_{\alpha\beta_i}|\geq ko_p(\frac{1}{\sqrt{\nw}})\bigr)\leq \frac{1}{k^2}  
\end{align}
Denote $A_i$ as the event that $|\widehat{PE}_{\alpha\beta_i}|\geq ko_p(\frac{1}{\sqrt{\nw}})$. By Boole's inequality, the union bound of the series for $i \in \{i: i\ns+2\nw-1<T^j_c\}$ is given by:
\begin{align}
P\bigl(\cup_{i=0}^{\lfloor\frac{T^j_c-2\nw+1}{\ns}\rfloor}A_i\bigr)&\leq \sum_{i=0}^{\lfloor\frac{T^j_c-2\nw+1}{\ns}\rfloor}p\bigl(A_i\bigr)\leq\frac{\lfloor\frac{T^j_c-2\nw+1}{\ns}\rfloor+1}{k^2} \\
P\bigl(\cap_{i=0}^{\lfloor\frac{T^j_c-2\nw+1}{\ns}\rfloor}A_i^c\bigr)
&\geq 1-\frac{\lfloor\frac{T^j_c-2\nw+1}{\ns}\rfloor+1}{k^2} 
\end{align}
That is, the probability that the maximum value of $\{\widehat{PE}_{\alpha\beta_i}\}_{i=0}^{i\ns+2\nw-1<T^j_c}$ is less than $ko_p(\frac{1}{\sqrt{\nw}})$ is larger than $1-\frac{\lfloor\frac{T^j_c-2\nw+1}{\ns}\rfloor+1}{k^2}$.
The same results hold for $\forall i \in \{i: i\ns\geq T^j_c\}$.

Let $\ns=\bs\log \tsub$, $k^2=\nw$ and make $\nw$ a proportion of $\tsub$, which can be denoted by $\frac{\tsub}{\aw}$ where $\bs$ and $a_s$ are both constants. 

We have:
\begin{align}
P\bigl(\cap_{i=0}^{\lfloor\frac{T^j_c-\frac{2\tsub}{\aw}+1}{\bs\log \tsub}\rfloor}A_i^c\bigr)
> 1-\frac{\aw-2}{\bs\log \tsub }-\frac{\aw}{\tsub}
\end{align}
In other words, the likelihood that the maximum $PE_{\alpha\beta_i}$, encompassing all samples before $T^j_c$ or after $T^j_c$, is less than $o(1)$ is greater than $1-\frac{\aw-2}{\bs\log \tsub }-\frac{\aw}{\tsub}$.
More specifically:
\begin{align}
    P\bigl(\max_i \{\widehat{PE}_{\alpha\beta_i}\}_i < o_p(1)\bigr)>1-\frac{\aw-2}{\bs\log \tsub }-\frac{\aw}{\tsub}
\end{align}
\end{proof}

\begin{lemma}\label{PE close to PE hat}
Let $\{\widehat{PE}_{\alpha\beta_i}\}_i$ be the estimated PE series for one time series segments $X^j(\Lambda)\subsetneq \Xb^j \subsetneq V$ and $T^j_c$ denote the true change point in this time series segments. Under assumption \textbf{A6-A7}, we have that $\forall i \in \{i: T^j_c<i\ns+2\nw-1<T^j_c+\nw\}$,
\begin{align}
&p\bigl(PE_{\alpha\beta_i}-c_i-(\sqrt{\nw}-1)o_p(\frac{1}{\sqrt{\nw}})\leq\widehat{PE}_{\alpha\beta_i}\\&\leq PE_{\alpha\beta_i}+c_{i}+(\sqrt{\nw}+1)o_p(\frac{1}{\sqrt{\nw}}) \bigr)> 1-\frac{1}{\nw} 
\end{align}
where $c_i^2\coloneqq\lVert r_{\alpha\beta_i}\rVert_{\infty}^2+\frac{(1-\alpha\beta_i)^2\lVert r_{\alpha\beta_i}\rVert_{\infty}^4}{4}+\frac{\alpha^2\beta_i^2\lVert r_{\alpha\beta_i}\rVert_{\infty}^4}{4}$.
\end{lemma}
The window index $i$ satisfying $T^j_c<i\ns+2\nw-1<T^j_c+\nw$ guarantees that all the samples in $W^1_i$ are collected from $p(x)$ and samples in $W^1_2$ are from a mixture distribution $(1-\beta_i)p(x)+\beta_ip'(x)$. Lemma \ref{PE close to PE hat} states that the estimated error between the estimated PE series and the true PE value is smaller than a constant with probability $1-\frac{1}{\nw}$ if $\nw$ are larger than some threshold. In other words, $ \forall k>0, \exists N$ such that $\forall  \nw \geq N$:
\begin{align}
    P(|\widehat{PE}_{\alpha\beta_i}-PE_{\alpha\beta_i}|< k+c_i)>1-\frac{1}{\nw}
\end{align}
\begin{proof}

$\forall i \in \{i: T^j_c<in_s+2\nw-1<T^j_c+\nw\}$,  the asymptotic expectation and variance of $\widehat{PE}_{\alpha\beta_i}$ is given by:
\begin{align}
\mathbb{E}(\widehat{PE}_{\alpha\beta_i})&=PE_{\alpha\beta_i}+o_p(\frac{1}{\sqrt{\nw}})\label{expectation}\\
\mathbb{V}(\widehat{PE}_{\alpha\beta_i})&\leq \frac{\lVert r_{\alpha\beta_i}\rVert_{\infty}^2}{\nw}+\frac{(1-\alpha\beta_i)^2\lVert r_{\alpha\beta_i}\rVert_{\infty}^4}{4\nw}+\frac{\alpha^2\beta_i^2\lVert r_{\alpha\beta_i}\rVert_{\infty}^4}{4\nw}+o_p(\frac{1}{\nw}).
\end{align}

For simplicity, let $\sigma_{i}^2\coloneqq \frac{\lVert r_{\alpha\beta_i}\rVert_{\infty}^2}{\nw}+\frac{(1-\alpha\beta_i)^2\lVert r_{\alpha\beta_i}\rVert_{\infty}^4}{4\nw}+\frac{\alpha^2\beta_i^2\lVert r_{\alpha\beta_i}\rVert_{\infty}^4}{4\nw}.$

By Chebyshev's inequality, we have:
\begin{align}
p\bigl(|\widehat{PE}_{\alpha\beta_i}-PE_{\alpha\beta_i}-o_p(\frac{1}{\sqrt{\nw}})|\geq k(\sigma_{i}+o_p(\frac{1}{\sqrt{\nw}}))\bigr)\leq \frac{1}{k^2}  \\
p\bigl(|\widehat{PE}_{\alpha\beta_i}-PE_{\alpha\beta_i}-o_p(\frac{1}{\sqrt{\nw}})|< k(\sigma_{i}+o_p(\frac{1}{\sqrt{\nw}}))\bigr)\geq 1-\frac{1}{k^2}  \label{confidence interval}
\end{align}

Hence we have:
\begin{align}
&p\bigl(PE_{\alpha\beta_i}-k\sigma_{i}-(k-1)o_p(\frac{1}{\sqrt{\nw}})\leq\widehat{PE}_{\alpha\beta_i}\leq PE_{\alpha\beta_i}+k\sigma_{i}+(k+1)o_p(\frac{1}{\sqrt{\nw}}) \bigr)> 1-\frac{1}{k^2} 
\end{align}

Let $k=\sqrt{\nw}$ and make $\nw$ a proportion of $T$, which can be denoted by $\frac{T}{a_w}$ where $a_s$ is a constant. 

We have:
\begin{align}
&p\bigl(PE_{\alpha\beta_i}-c_i-(\sqrt{\nw}-1)o_p(\frac{1}{\sqrt{\nw}})\leq\widehat{PE}_{\alpha\beta_i}\leq PE_{\alpha\beta_i}+c_{i}+(\sqrt{\nw}+1)o_p(\frac{1}{\sqrt{\nw}}) \bigr)\\
&> 1-\frac{1}{\nw} 
\end{align}
where $c_i^2\coloneqq \lVert r_{\alpha\beta_i}\rVert_{\infty}^2+\frac{(1-\alpha\beta_i)^2\lVert r_{\alpha\beta_i}\rVert_{\infty}^4}{4}+\frac{\alpha^2\beta_i^2\lVert r_{\alpha\beta_i}\rVert_{\infty}^4}{4}.$

\end{proof}

Since we only discuss the discrete distribution, therefore we can assume that there are $k$ realizations of any variable $X^j_t$ from $1$ to $k$. 
For simplicity, we denote $p(x=k)$ as $p_h$ and $p'(x=k)$ as $p'_h$.
\begin{lemma}\label{increasing}
Let $\{PE_{\alpha\beta_i}\}_i$ be the estimated PE series for one time series segments $X^j(\Lambda)\subsetneq \Xb^j \subsetneq V$ and $T^j_c$ denote the true change point in this time series segments. Under assumption \textbf{A6-A7}, we have that $\forall i \in \{i: T^j_c<i\ns+2\nw-1<T^j_c+\nw\}$,
\begin{itemize}
    \item $PE_{\alpha\beta_i}>0$
    \item $PE_{\alpha\beta_i}$ is a monotonically increasing function regarding $i$ (or $\beta_i$).
    \item $\max PE_{\alpha\beta_i}=\frac{1}{2}\sum^{k}_{h=1}\frac{p_h^2}{(1-\alpha\beta^{*})p_h+\alpha\beta^{*}p'_h}-\frac{1}{2}$ where $\beta^{*}$ is the largest proportion of samples from $p'(x)$ over those second windows $\{W^2_i\}_{i: T^j_c<i\ns+2\nw-1<T^j_c+\nw}$. For appropriate $\nw$ and $\ns$, $\max PE_{\alpha\beta_i}$ can achieve the maximum when all samples of $W^1_i$ are from $p(x)$ and all samples of $W^2_i$ are from $p'(x)$, that is, $\max_{i} \beta_i = 1$.
\end{itemize}
\end{lemma}
Lemma \ref{increasing} suggests that the true PE series derived from the sequence of sliding windows is a monotonically increasing function with respect to $\beta_i$ if the samples in the second half of the sliding windows are from a mixture distribution, indicating that the change point is located within the second half of the sliding windows. Similarly, employing the same reasoning, the true PE series becomes a monotonically decreasing function with respect to $\beta_i$ if the samples in the first half of the windows are from a mixture distribution, that is, $i \in \{i: T^j_c+\nw <i\ns+2\nw-1<T^j_c+2\nw\}$. In that case, $\beta_i$ represents the proportion of samples from $p(x)$ over the first half windows.
\begin{proof}
Treat $PE_{\alpha\beta_i}$ as a function of $\alpha\coloneqq \alpha\beta_i$. 
\begin{align}
PE_{\alpha\beta_i} & =
\frac{1}{2}\mathbb{E}_{p(x)}[r_{_{\alpha\beta_i}}(x)]-\frac{1}{2}\\
&=\frac{1}{2}\sum^{k}_{h=1}\frac{p_h}{(1-\alpha)p_h+\alpha p'_h}p_h-\frac{1}{2}\\
&=\frac{1}{2}\sum^{k}_{h=1}\frac{p_h}{1+(\frac{p'_h}{p_h}-1)\alpha}-\frac{1}{2}
\end{align}
We have:
\begin{align}
f(\alpha) &\coloneqq \sum^{k}_{h=1}\frac{p_h}{1+(\frac{p'_h}{p_h}-1)\alpha}\\
f'(\alpha) &= \sum^{k}_{h=1}\frac{-p_h(\frac{p'_h}{p_h}-1)}{(1+(\frac{p'_h}{p_h}-1)\alpha)^2}\\
f''(\alpha) 
&= \sum^{k}_{h=1}\frac{2p_h(\frac{p'_h}{p_h}-1)^2}{(1+(\frac{p'_h}{p_h}-1)\alpha)^3}
\end{align}
Since $\frac{p'_h}{p_h}>0$ and $0\leq \alpha \leq 1$, we have $1+(\frac{p'_h}{p_h}-1)\alpha>0$ for any $j\in [k]$, hence $f''(\alpha)>0$, leading to an increasing function $f'(\alpha)$. 
\begin{align}
    f'(\alpha)\geq f'(0) = \sum^{k}_{h=1}(p_h-p'_h)=0
\end{align}
Hence $f(\alpha)$ is an increasing function in terms of $\alpha$ and 
\begin{align}
    f(\alpha)\geq f(0) = \sum^{k}_{h=1}p_h = 1 \text{, with }\alpha\in [0,1]
\end{align}    
\end{proof}
\begin{theorem}\label{Final}
Let $\{\widehat{PE}_{\alpha\beta_i}\}_i$ be the estimated PE series for one time series segments $X^j(\Lambda)\subsetneq \Xb^j \subsetneq V$ and $T^j_{c}$ denote the true change point in this time series segments. $\widehat{T}_{c}^j$ denotes the estimator of $T^j_{c}$ obtained by:
\begin{align}
    \widehat{T}_{c}^j=\arg_i\max\{\widehat{PE}_{\alpha\beta_i}\}_i
\end{align}
Under assumption \textbf{A6-A9}, we have that given large enough $\nw$, $\forall i \in [\tau_\text{max}+1,T]$
\begin{align}
     P\bigl(|\widehat{T}_{c}^j-T^j_c| < 2\nw\bigr)>(1-\frac{\aw-2}{\bs\log \tsub }-\frac{\aw}{\tsub})(1-\frac{1}{\nw})
\end{align}
where $\bs=\lfloor \frac{\log \tsub}{\ns}\rfloor$ and $\aw=\lceil\frac{\tsub}{\nw}\rceil$.
\end{theorem}
\begin{proof}
From Theorem \ref{zero},  $ \forall k_1>0, \exists N_1$ such that $\forall   \nw \geq N_1$, we have $\forall i_1 \in \{i_1: i_1\ns+2\nw-1<T_{c}\}$:
    \begin{align}
        &P\bigl(\max_{i_1} \{\widehat{PE}_{\alpha\beta_{i_1}}\}_{i_1} < k_1\bigr)>1-\frac{\aw-2}{\bs\log \tsub }-\frac{\aw}{\tsub}\label{theorem 4.1}
     \end{align}
From Lemma \ref{PE close to PE hat}, $ \forall k_2>0, \exists N_2$ such that $\forall  \nw \geq N_2$, we have $\forall i_2 \in \{i: T^j_c<i_2\ns+2\nw-1<T^j_c+\nw\}$:
    \begin{align}
        &P(|\widehat{PE}_{\alpha\beta_{i_2}}-PE_{\alpha\beta_{i_2}}|< k_2+c_{i_2})>1-\frac{1}{\nw}
    \end{align}
With Assumption $\textbf{A9}$, $\exists i_2$ such that $PE_{\alpha\beta_{i_2}}>c_{i_2}$, hence $\exists k_1,k_2>0$:
\begin{align}
    PE_{\alpha\beta_{i_2}}\geq c_{i_2}+k_2+k_1
\end{align}
Hence we can replace $k_1$ with $PE_{\alpha\beta_{i_2}}-c_{i_2}-k_2$ in Eq.\ref{theorem 4.1} and then we have:
\begin{align}
  &P\bigl(\max_{i_1} \{\widehat{PE}_{\alpha\beta_{i_1}}\}_{i_1} < PE_{\alpha\beta_{i_2}}-c_{i_2}-k_2\bigr)>1-\frac{\aw-2}{\bs\log \tsub }-\frac{\aw}{\tsub}\label{new theorem 4.1}  
\end{align}
From Lemma \ref{PE close to PE hat}, we have:
\begin{align}
    PE_{\alpha\beta_{i_2}}-k_2-c_{i_2}<\widehat{PE}_{\alpha\beta_{i_2}}<PE_{\alpha\beta_{i_2}}+k_2+c_{i_2}\label{new lemma}
\end{align}
holds with probability $1-\frac{1}{\nw}$.

As samples in $W_{i_1}$ and $W_{i_2}$ are IID, the events in Eq.\ref{new theorem 4.1} and Eq.\ref{new lemma} are independent, resulting in:
\begin{align}
    P(\max_{i_1} \{\widehat{PE}_{\alpha\beta_{i_1}}\}_{i_1}<\widehat{PE}_{\alpha\beta_{i_2}})>(1-\frac{\aw-2}{\bs\log \tsub }-\frac{\aw}{\tsub})(1-\frac{1}{\nw})
\end{align}
for $\forall i_1 \in \{i_1: i_1\ns+2\nw-1<T_{c}\}$ and $\exists i_2 \in \{i: T^j_c<i_2\ns+2\nw-1<T^j_c+\nw\}$.

From Lemma \ref{increasing}, we know that $\exists \beta^*$ such that $PE_{\alpha\beta^{*}}=\max PE_{\alpha\beta_{i_2}}$. Denote $i^*_2 = \arg_{i_2}\max PE_{\alpha\beta_{i_2}}$. 

Since $PE_{\alpha\beta_{i_2}}$ is a monotonically increasing function regarding $\beta_{i_2}$, then the true change point $T^j_c$ must satisfying:
\begin{align}
 PE_{\alpha\beta_{\frac{T^j_c-\nw}{\ns}}}\geq PE_{\alpha\beta^{*}} \geq\widehat{PE}_{\alpha\beta_{i_1}}\label{PE Tc},
\end{align}
$\forall i_1 \in \{i: i_1\ns+2\nw-1<T^j_c\}$ as $\ns$, $\nw$ and $T$ has jointly determined $\beta_i$.
The proof for situations that $\{i_1: i_1\ns\geq T^j_c\}$ and $i_2 \in \{i: T^j_c+\nw <i\ns+2\nw-1<T^j_c+2\nw\}$ follows the same logic. 

With Eq.\ref{PE Tc}, we finally have:
\begin{align}
     P\bigl(|\widehat{T}_{c}^j-T^j_c| < 2\nw\bigr)>(1-\frac{\aw-2}{\bs\log \tsub }-\frac{\aw}{\tsub})(1-\frac{1}{\nw})
\end{align}

\end{proof}
\section{Related work}\label{app:related work}

Identifying causal relationships in stationary time series is more challenging than in IID samples, and finding these relationships in non-stationary time series is even harder. In our scenario, when change points occur, the $\textit{Mechanism-Shift}$ SCM leads to a non-stationary time series. Each mechanism in this model corresponds to a stationary time series. Once we accurately detect the change points, we can divide the non-stationary time series into multiple stationary components. This simplifies the task of finding causal relationships in non-stationary time series to that of discovering them in stationary time series. Therefore, our algorithm can be viewed as a method for causal discovery in non-stationary time series, driven by change point detection.

Numerous efforts have been made recently to adapt causal discovery algorithms originally designed for IID data to work with stationary time series data, such as \citep{runge2019detecting}, \citep{entner2010causal}, \citep{hyvarinen2010estimation}, \citep{peters2013causal} and \citep{pamfil2020dynotears}.

Most causal discovery methods for non-stationary time series rely on parametric models. Examples include the vector autoregressive model used in \citep{gong2015discovering} and \citep{malinsky2019learning}, as well as linear and non-linear state-space models in \citep{huang2019causal}. Other approaches focus on linear causal relationships, as seen in \citep{saggioro2020reconstructing}. A non-parametric method called CD-NOD, described in \citep{huang2020causal}, can identify time-lagged causal relationships in non-stationary time series. Additionally, \citep{fujiwara2023causal} introduced the JIT-LiNGAM algorithm, which combines the LiNGAM approach with a JIT framework to create a local approximated linear causal model for non-linear and non-stationary data. In \citep{gao2023causal}, it is assumed that the underlying causal mechanisms of each time series change sequentially and periodically.

We would like to thoroughly clarify the differences between the setting in \citep{gao2023causal} and our paper to prevent any potential confusion.

Our work Causal-RuLSIF and the previous work of PCMCI$_\Omega$ in \citep{gao2023causal} are designed for distinct settings. We highlight two significant differences below.

\textbf{Periodicity} In PCMCI$_\Omega$, the Structural Causal Model (SCM) underlying the non-stationary time series is defined for a specific type of non-stationarity called semi-stationary time series. These are characterized by a finite number of different causal mechanisms occurring periodically over time. In contrast, our work does not assume such periodicity but rather a change point in the causal mechanism.

For example, consider two different causal mechanisms, denoted as $A$ and $B$. The shifting causal mechanisms over time can be expressed as $[A, B, A, B, A, B, \cdots]$. PCMCI$_\Omega$ can detect the periodicity of such changes. On the other hand, Causal-RuLSIF is designed for scenarios where a time series $\textbf{X}^j \in V$ transitions from one causal mechanism to another without periodicity, such as $[A, A, A, \cdots, A, B, B, \ldots, B]$. Causal-RuLSIF identifies the change point where the mechanism shifts from $A$ to $B$. While Causal-RuLSIF can not be applied to settings like $[A, B, A, B, A, B, \cdots]$, PCMCI$_\Omega$ also fails to handle settings like $[A, A, A, \cdots, A, B, B, \cdots, B]$.

\textbf{Mechanism Change} Additionally, PCMCI$_\Omega$ relies on the Hard Mechanism Change Assumption, which requires that the incoming edges of causal mechanisms $A$ and $B$ do not share time-invariant parent sets, reflecting a time-variant structure. In contrast, Causal-RuLSIF does not impose constraints on the causal graph structure; it directly assesses changes in the conditional distribution using a distribution divergence score. As a result, Causal-RuLSIF can handle both hard and soft mechanism changes.





\section{Generate Binary-valued Time Series}\label{app:data generation}
The generation process has three steps.
\begin{enumerate}
  \item Determine time series length $T$, number of time series $n$ of the multivariate time series, data domain $D$, maximum time lag $\tau_{\max}$. Randomly generate $T^j_{c}$, $j\in[n]$, satisfying Definition \ref{def:Mechanism-Shift} and Assumption $\textbf{A6-A7}$. 
  Additionally, in order to guarantee enough samples in each time series segment, we also control the size of the union parent set, that is, $|\SPA{X^j_t}|$.

  With $n$, $\tau_{\max}$ and the number of change points for each component time series $c^j$, one binary edge array with dimension $[n,c^j+1,n,\tau_{\max}+1]$ is randomly generated, where the first $n$ denotes the index of parent variable, $c^j$ denotes the number of change point, the second $n$ represents the index of target variable (child) and $\tau_max+1$ denotes the time lag. With Assumption $\textbf{A7}$, $c^j=1$ for all $j\in [n]$. $1$ in this binary edge array means that there is an edge from the parent variable to the child variable with the corresponding time lag; $0$ means that there is no cause-effect between two corresponding variables.
  
For a soft mechanism change, the parent set remains the same before and after the change point, without intersecting other change points, so the binary edge arrays representing the different causal mechanisms should be identical. In contrast, for a hard mechanism change, at least two of these binary edge arrays must differ.
  
  \item  With the randomly generated binary edge array controlled by $|\SPA{X^j_t}|$, a full causal graph for this $n$-variate time series is obtained. Given this edge array and data domain $D$, parent configuration matrices $\{\Sigma^j\}_{j\in[n]}$ are obtained. Corresponding to $\{\Sigma^j\}_{j\in[n]}$, conditional probability tables (CPTs) are randomly generated. 
  

  \item Fill the starting points $\{X^j_{t\leq \tau_{\max}}\}_{j\in[n]}$ with randomly generated binary data. Starting from time point $t>\tau_{\max}$, generate vector $\Xb_{t}$ over time according to the CPTs, until $t$ achieves $T$.
\end{enumerate}

Please note that since our algorithm can be extended to handle multiple change points, the generation process can also be easily extended to accommodate such settings.

\section{Additional Experiments}\label{app:simulations}
From Fig.\ref{fig3:a} and Fig.\ref{fig3:b}, it's apparent that a large $\nw$ may affect the performance of Causal-RuLSIF. This issue still stems from sample efficiency. The total number of windows for a specific time series segment $X^j(\Lambda)$ is given by $\lfloor\frac{|\tsub|-2\nw}{\ns}\rfloor+1$, where $\tsub = |X^j(\Lambda)|$. Increasing $\nw$ decreases the total number of windows, resulting in less information collected, as smaller $\nw$ is sufficient to estimate $PE$ scores with kernel functions.
\begin{figure}[t!]
\centering     
\subfigure[Influence of $\nw$ on estimation error]{\label{fig3:a}\includegraphics[height=55mm,width=65mm]{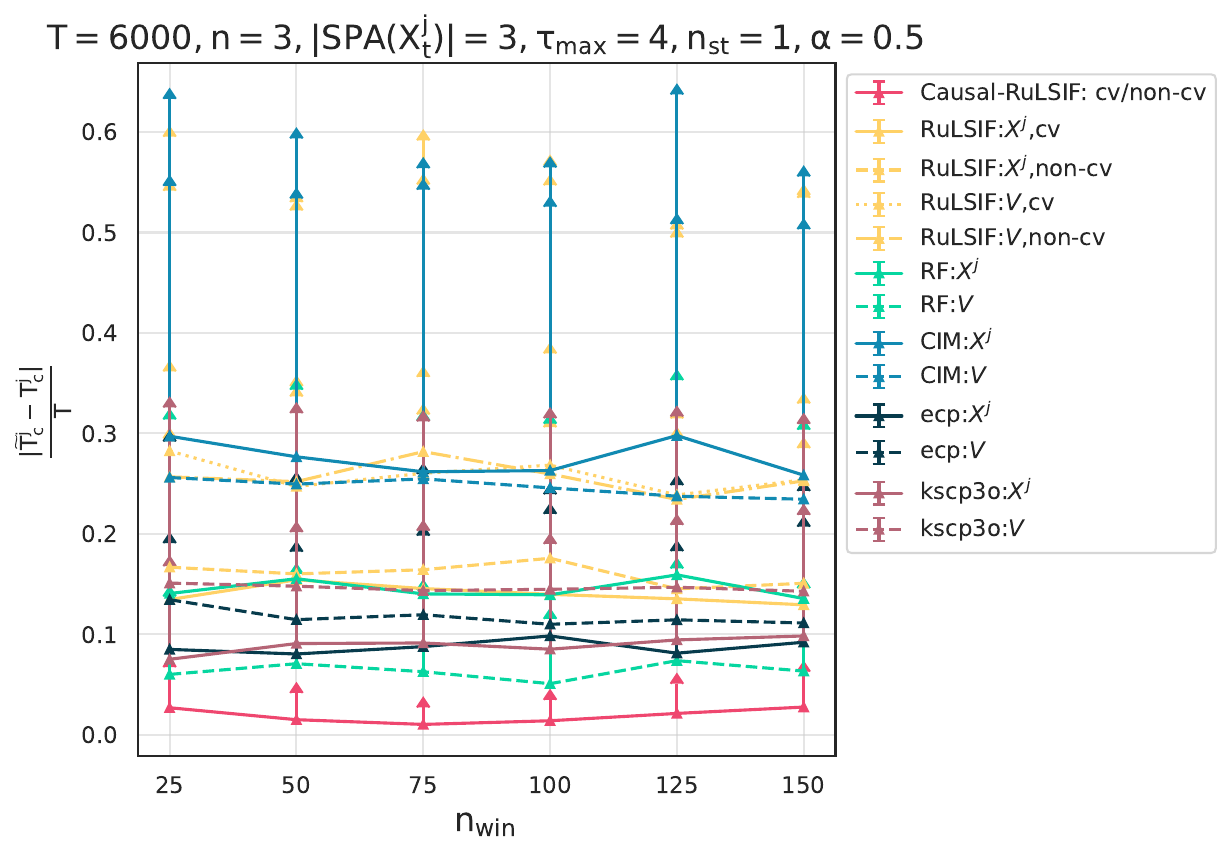}}
\subfigure[ROC curves for different $\nw$]{\label{fig3:b}\includegraphics[height=55mm,width=73mm]{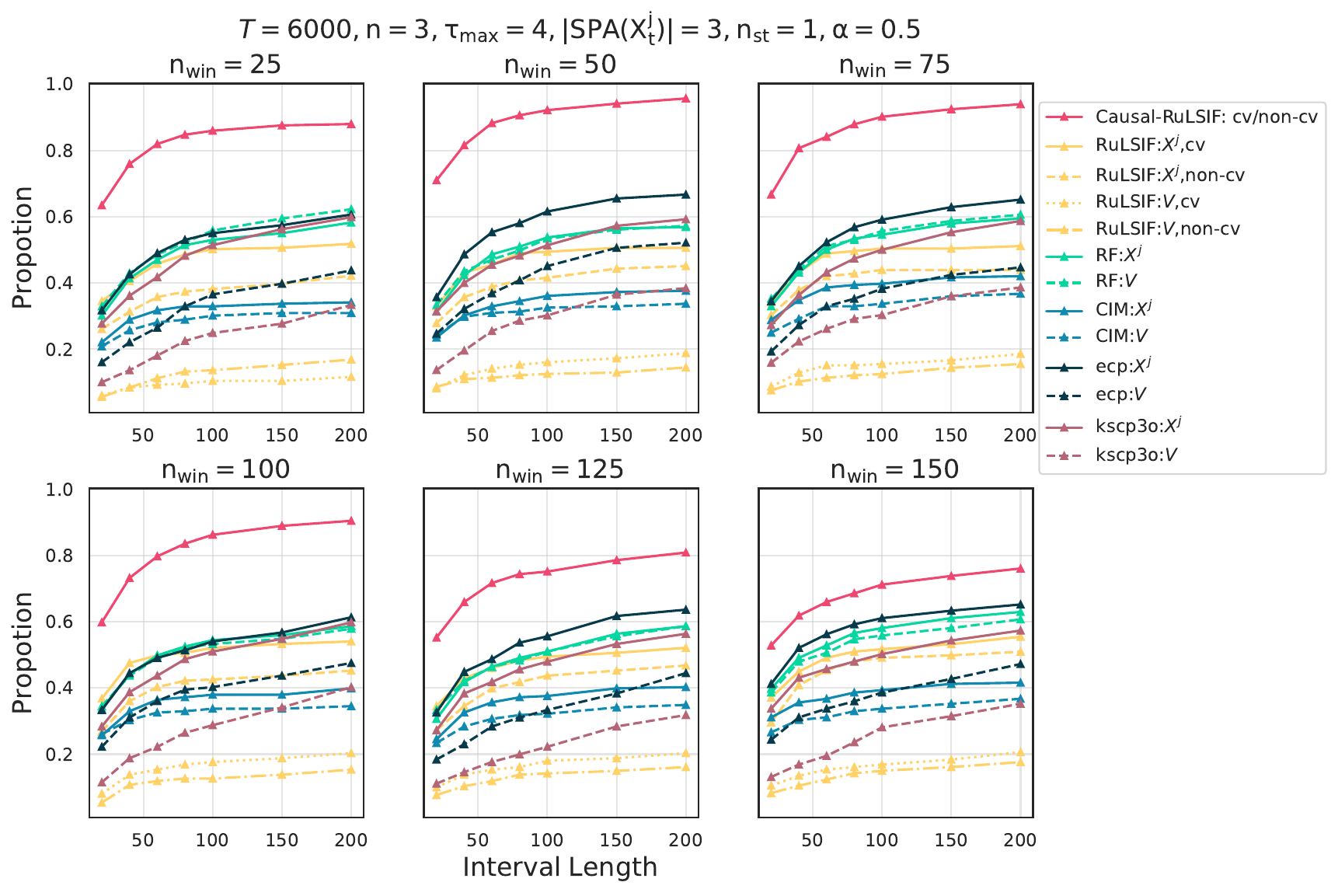}}
\caption{Causal-RuLSIF is tested on 3-multivariate binary time series with $T=6000,\tau_{\max}=4,|\SPA(X^j_t)|=3,\ns=1$ with \emph{hard mechanism change}. Every line with a different color corresponds to a different algorithm and different linestyle corresponds to a different setting. $X^j$ in the legend means the algorithm is applied to each component time series while $V$ means the algorithm is used for the whole $n$-variate time series $V$. Every marker corresponds to the average error rate or average accuracy rate over 100 random trials. The error bar represents the standard error for the averaged statistics. a) Influence of $\nw$ on estimation error $\frac{|\widetilde{T}^j_c-T^j_c|}{T}$. b) ROC curves for different $\nw$.}
\end{figure}

As demonstrated in Fig.\ref{fig5:a} and Fig.\ref{fig5:b}, the performance of Causal-RuLSIF remains consistent regardless of the distance between the change points $T^1_c$ of $\Xb^1$ and $T^2_c$ of $\Xb^2$.
\begin{figure}[t!]
\centering     
\subfigure[Influence of $|T^1_c-T^2_c|$ on estimation error]{\label{fig5:a}\includegraphics[height=52mm,width=70mm]{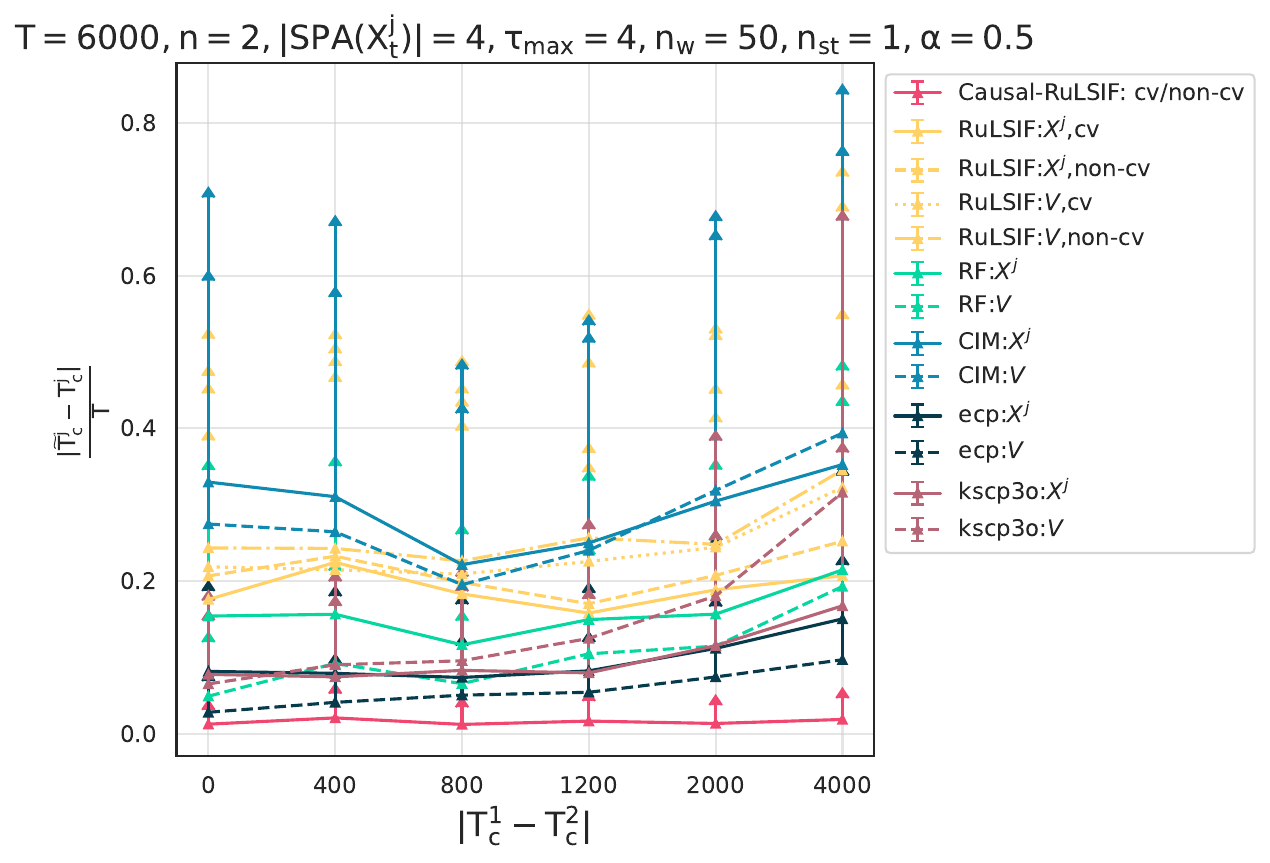}}
\subfigure[ROC curves for different $|T^1_c-T^2_c|$]{\label{fig5:b}\includegraphics[height=52mm,width=73mm]{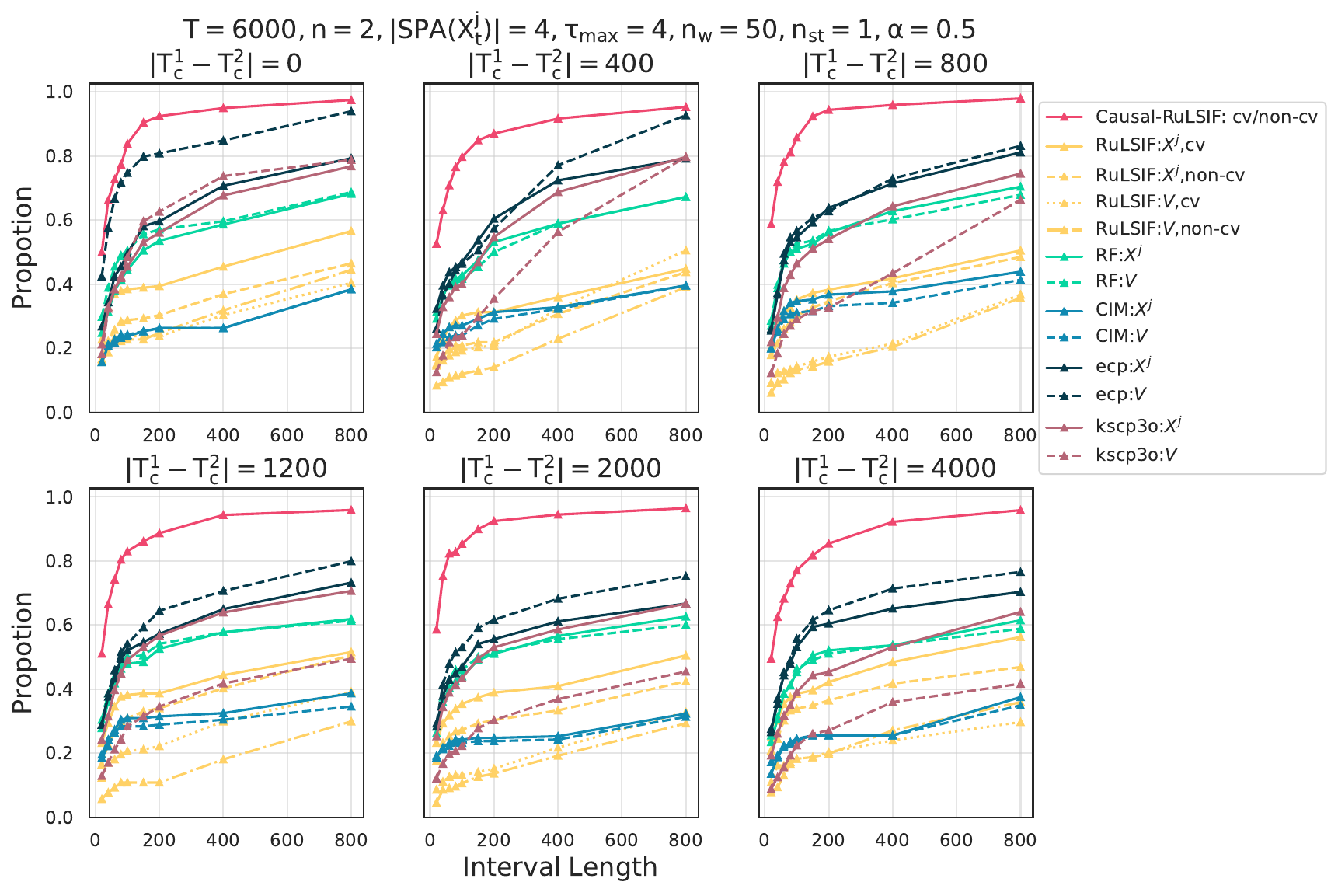}}
\caption{Causal-RuLSIF is tested on 3-multivariate binary time series with $T=6000,\tau_{\max}=4,|\SPA(X^j_t)|=3,\nw=50,\ns=1$ with \emph{soft mechanism change}. Every line with a different color corresponds to a different algorithm and different linestyle corresponds to a different setting. $X^j$ in the legend means the algorithm is applied to each component time series while $V$ means the algorithm is used for the whole $n$-variate time series $V$. Every marker corresponds to the average error rate or average accuracy rate over 100 random trials. The error bar represents the standard error for the averaged statistics. a) Influence of $|T^1_c-T^2_c|$ on estimation error $\frac{|\widetilde{T}^j_c-T^j_c|}{T}$. b) ROC curves for different $|T^1_c-T^2_c|$.}
\end{figure}

Fig. \ref{fig 6} shows the running time needed for each algorithm in seconds. Cross-validation techniques is time-consuming. It is beneficial that cross-validation is not needed for binary time series data for Causal-ReLSIF.
Additionally, in order to have a better illustration of implementing dynamic divergence measurement on time series segments, we have conducted experiments on randomly generated binary time series and plot the PE divergence series obtained from the time series segments. 

\begin{figure}[t!]
\centering     
\subfigure[Runtime (in sec.) for algorithms]{\label{fig 6}\includegraphics[height=52mm,width=68mm]{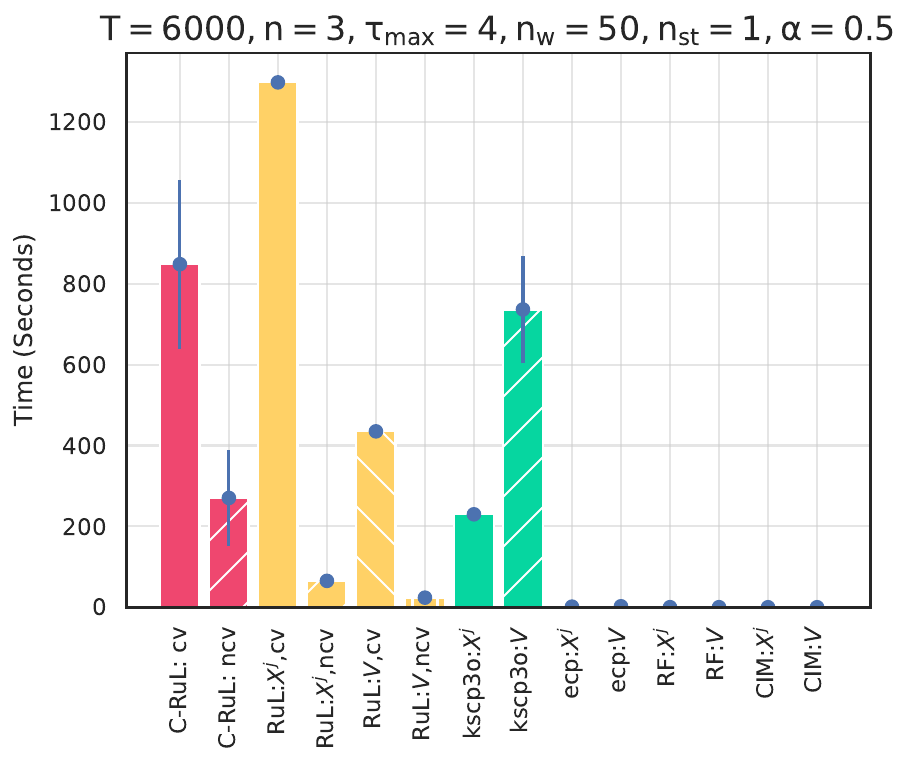}}
\subfigure[One estimated PE divergence series obtained on one time series segment]{\label{fig:PE score}\includegraphics[height=52mm,width=80mm]{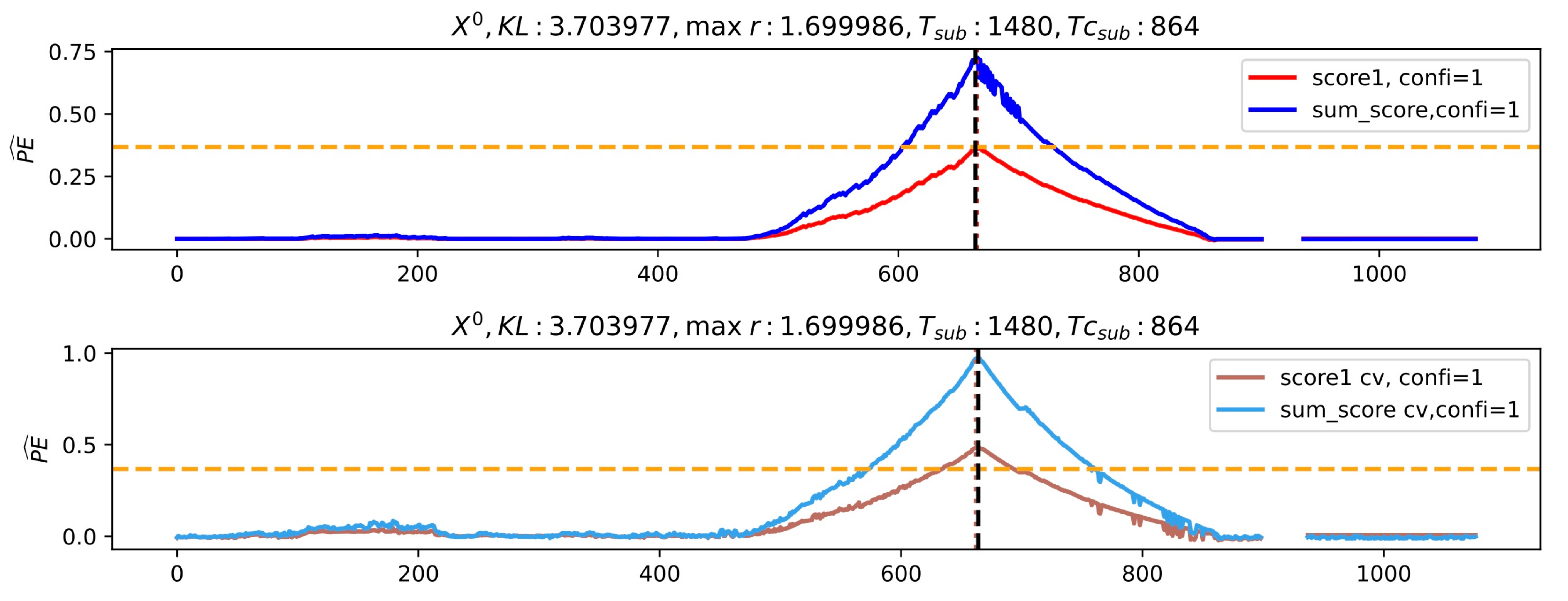}}
\caption{Runtime and one illustration of time series segment.}
\end{figure}

The estimated PE divergence series shown in Fig.~\ref{fig:PE score} for one time series segment verifies Theorem \ref{main: theorem 4.1}, Lemma \ref{increasing} and Theorem \ref{main: theorem 4.2} empirically. For window index $i$ smaller than around 500 or larger than around $900$, $\hat{PE}_i$ score is very close to zero. For window index $i$ between $500$ and $900$, $\widehat{PE}_i$ first monotonically increases and then monotonically decreases over the window index $i$ and the maximum value is achieved at $\widehat{T}_c=864$, which equals to the true change point in this time series segment. Furthermore, the maximum of the PE score line is equal to the true PE value, which is denoted by the dashed orange horizontal line. In the title of the subgraphs, $KL$ represents the Kullback-Leibler (KL) divergence between the two conditional distributions of $X^0$ before and after the change point, while $\max r$ indicates the maximum density ratio $\frac{p(x)}{(1-\alpha\beta_i)p(x)+\alpha\beta_ip'(x)}$ over all $x$ in its domain.

\section{Sample Efficiency: k-PC and Top-K parents}\label{app:k-PC and Top-K}
As noted in the main paper, sample efficiency is a fundamental issue in our algorithm. The effective sample size of each time series segment, $|T_\text{sub}| \approx T/|s^{|\widehat{S\Pa}(X^j_t)|}|$, decreases exponentially with domain size $s$ and parent size $\widehat{S\Pa}(X^j_t)$, whereas the baseline method has a sample size of $T$. This is a necessary trade-off if we focus on the $\textbf{causal mechanims}$ shift rather than the shift in $\textbf{joint distributions}$.

In practice, we can address this issue through two approaches: the k-PC algorithm and the top-K causal strength selection.

Specifically, we can improve the power of the CI tests by applying the k-PC algorithm from \citep{kocaoglu2024characterization} during the PC stage of PCMCI, which restricts the size of the conditioning set to k. Furthermore, when generating time series segments with $\SPA(X^j_t)$, we can directly limit the size of $\SPA(X^j_t)$ by selecting only the parents with the top $K$ causal strengths based on the statistics from the MCI tests in PCMCI. As noted in \citep{runge2019detecting}, the MCI test statistic can be interpreted as a measure of causal strength, enabling the meaningful ranking of causal links in large-scale studies.

Given a large multivariate time series dataset with $n = 4, 6, 8, 10$ univariate time series, Fig.~\ref{fig:topk error} and ~\ref{fig:topk pro} demonstrate that more accurate time series segments result in more accurate estimates of $\widehat{T^j_c}$, highlighting the need for causal-driven change point detection under the \textit{Mechanism-shift} SCM. Our proposed algorithm, Causal-RuLSIF with the Top 1 parent, may perform worse than some baselines due to sample efficiency limitations. Additionally, in this experiment, when no significant parent is identified during causal discovery, we select the parent with the highest causal influence, even if it lacks significance, which can further decrease the performance of the Top 1 parent method. However, Causal-RuLSIF with the Top 3 parents outperforms other baselines.
\begin{figure}[t!]
\centering     
\subfigure[Influence of selecting top-K parents on estimation error]{\label{fig:topk error}\includegraphics[height=52mm,width=70mm]{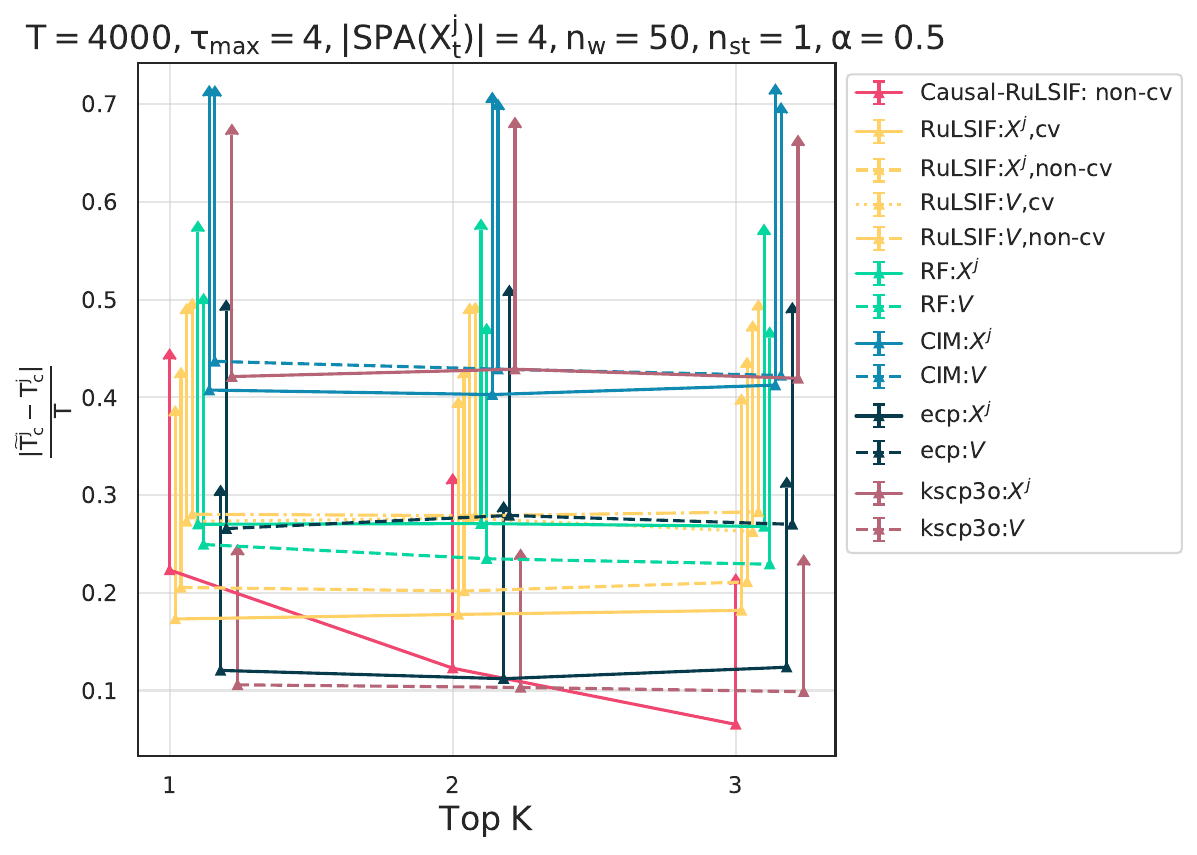}}
\subfigure[ROC curves for different $n$-variate time series]{\label{fig:topk pro}\includegraphics[height=52mm,width=73mm]{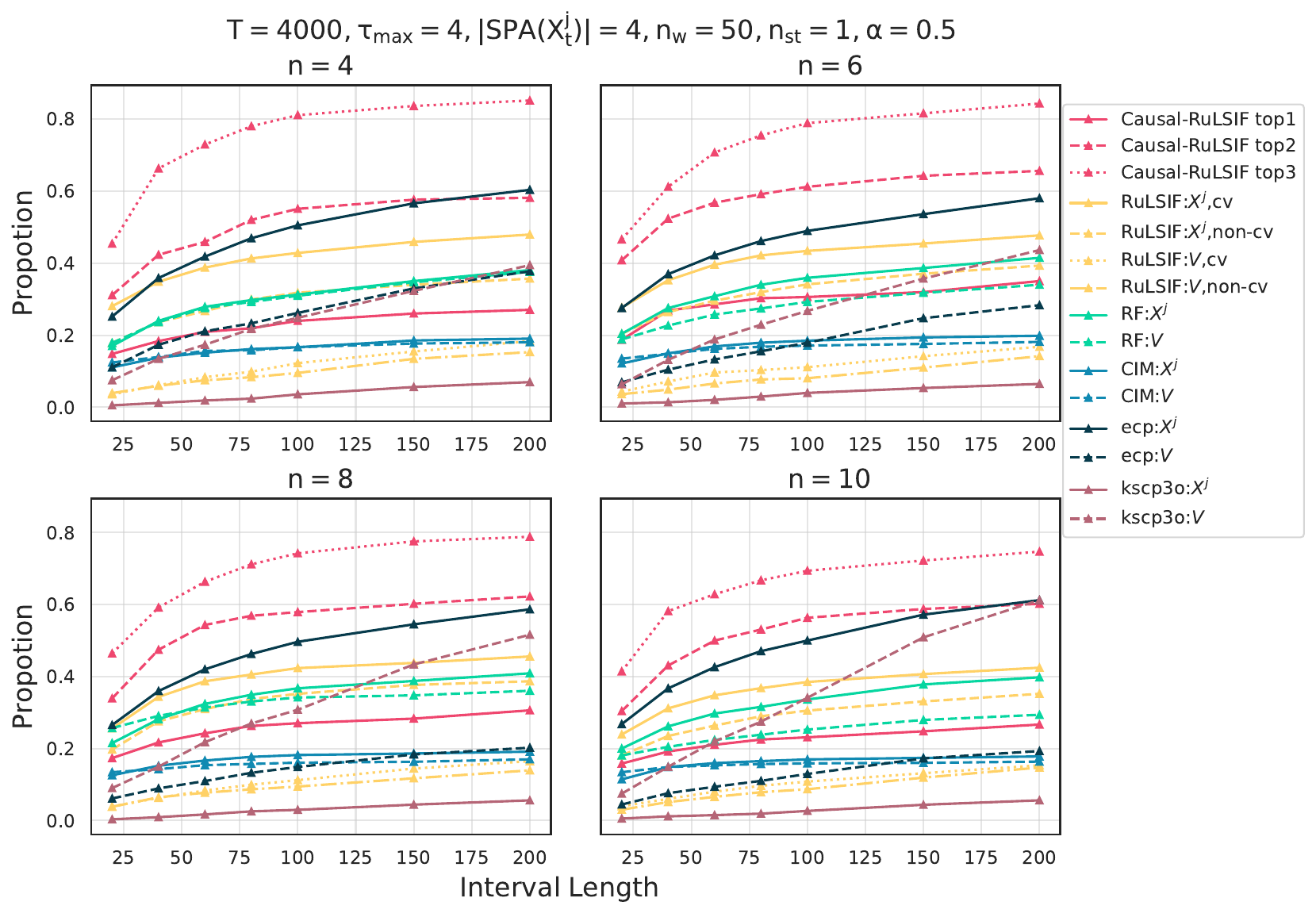}}
\caption{Causal-RuLSIF is tested on $n$-variate time series where $T=6000,\tau_{\max}=4,|\SPA(X^j_t)|=4,\ns=1, n=[4,6,8,10]$ with \emph{soft mechanism change}. The domain size $s=3$. Every line with a different color corresponds to a different algorithm and different linestyle corresponds to a different setting. $X^j$ in the legend means the algorithm is applied to each component time series while $V$ means the algorithm is used for the whole $n$-variate time series $V$. Every marker corresponds to the average error rate or accuracy rate over 50 random trials. The error bar represents the standard error for the averaged statistics. a) Influence of Top-$K$ parents on estimation error $\frac{|\widetilde{T}^j_c-T^j_c|}{T}$. b) ROC curves for different $n$-variate time series.}
\end{figure}

\section{Multiple Change Points}\label{app:multiple change points}

Our algorithm can be readily extended to handle multiple change point problems, subject to \textbf{Assumption A7}.

If the temporal distance between two consecutive change points in time series segments exceeds the window size, i.e., $\Delta_c > 2\nw$, the theoretical guarantee for detecting each individual change point, as discussed in the main paper, remains valid.

Fig.~\ref{fig: mcp error} and Fig.~\ref{fig: mcp pro} illustrate the performance of algorithms on time series where each univariate time series contains at most $4$ change points, specifically, $c^j \leq 4$ for $j \in [3]$. Our algorithm outperforms the other baselines.

\begin{figure}[t!]
\centering     
\subfigure[Estimation error]{\label{fig: mcp error}\includegraphics[height=52mm,width=70mm]{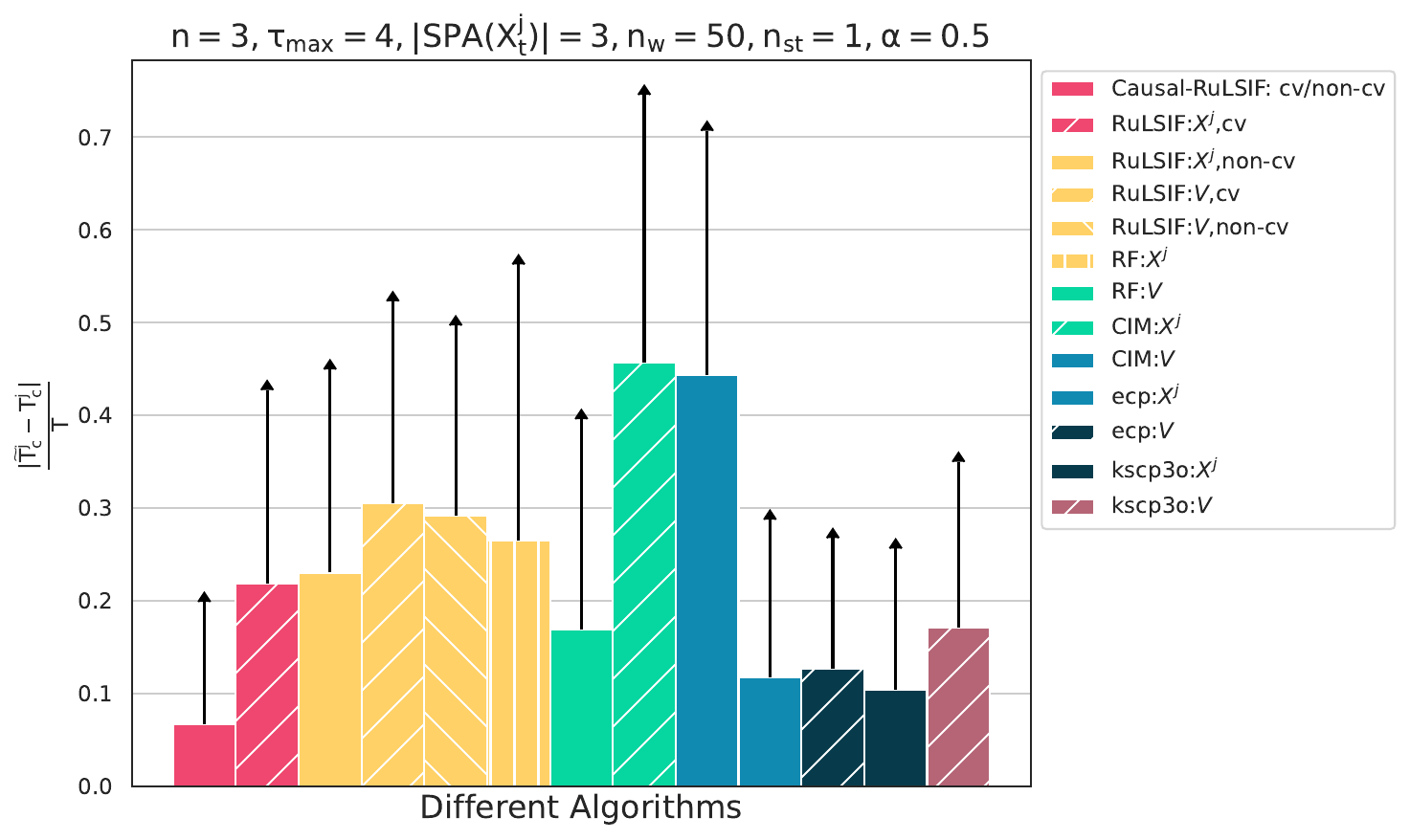}}
\subfigure[ROC curves]{\label{fig: mcp pro}\includegraphics[height=52mm,width=73mm]{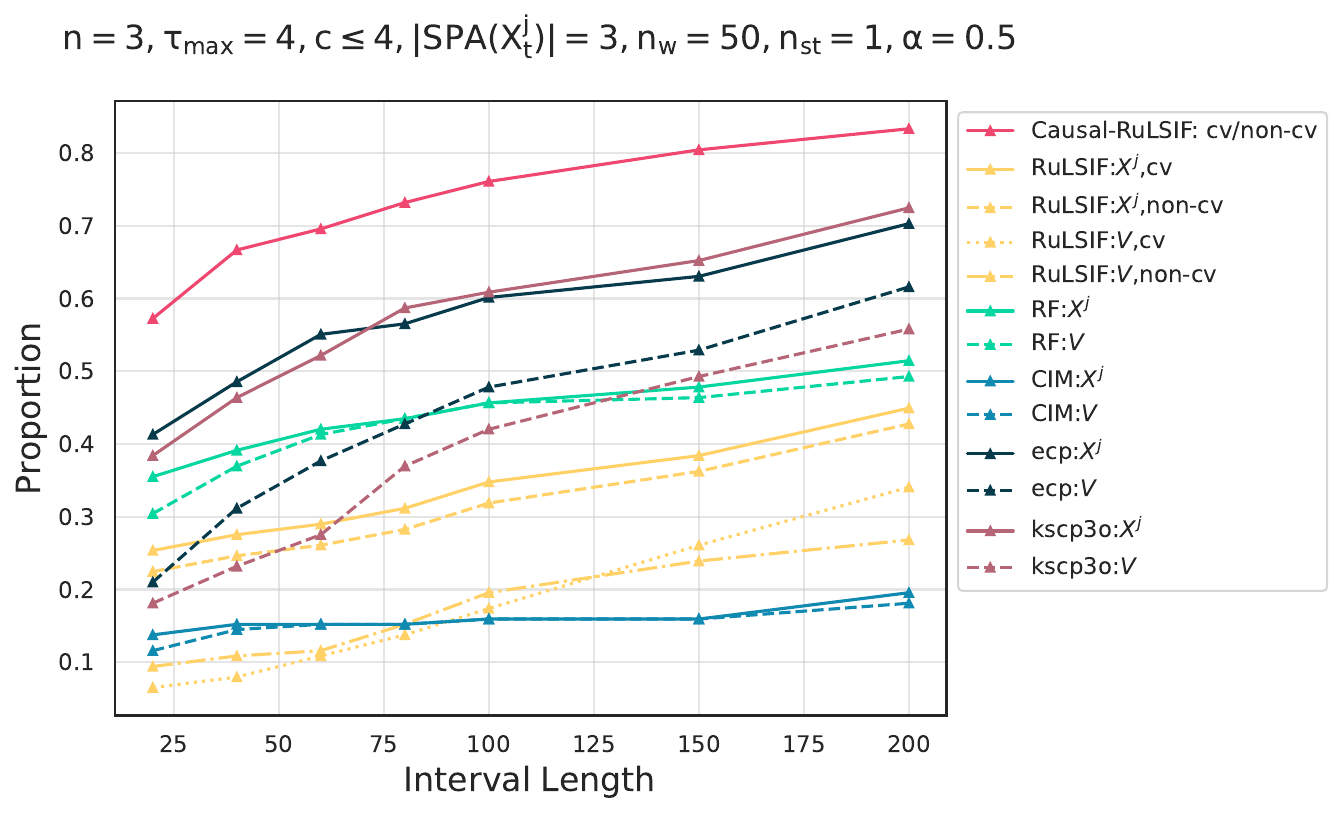}}
\caption{Causal-RuLSIF is tested on 3-multivariate binary time series with $T=6000,\tau_{\max}=4,|\SPA(X^j_t)|=3,\nw=50,\ns=1$ with \emph{soft mechanism change}. Every line with a different color corresponds to a different algorithm and different linestyle corresponds to a different setting. $X^j$ in the legend means the algorithm is applied to each component time series while $V$ means the algorithm is used for the whole $n$-variate time series $V$. Every marker corresponds to the average error rate or average accuracy rate over 50 random trials. The error bar represents the standard error for the averaged statistics. a) Estimation error with different algorithms. b) ROC curves for different algorithms.}
\end{figure}

\section{Case Study: A Cautionary Counterexample}\label{app:case study}

When applying our causal-driven algorithm to a real dataset, it is crucial to ensure that there are causal relationships among the \(n\)-variate time series. If no causal relationships exist within the \(n\)-variate time series, we do not recommend using our algorithm. This is because the identified causal relationships may not be meaningful, and the efficiency of the sample will diminish due to the presence of misleading or meaningless causal relationships.

Additionally, please carefully check whether all the assumptions in section~\ref{app:assumptions} are met. 

To illustrate this warning, we conduct a case study on a human activity dataset, as used in \citep{alanqary2021change}. This dataset was collected using a portable three-axis accelerometer on human beings. The three dimensions of the time series reflect the acceleration recorded by the device along the x, y, and z axes. The change points indicate transitions between six activities: standing, walking, jogging, skipping, climbing up stairs, and climbing down stairs. We select a sequence of length $1000$ that contains only one change point. 

First, common sense suggests that there are no clear causal relationships among this $3$-variate time series. Second, \textbf{Assumption A1} is violated since the three axes are affected by human movement, which implies the existence of an unobserved confounder.

We continue to apply causal-RuLSIF to this dataset to highlight its inappropriateness. As illustrated in Fig.~\ref{fig:human}, while the estimated change point \(\widetilde{T}^j_c\) on the \(y\)-axis and \(z\)-axis is close to the true change point \(T^j_c\), this does not confirm the accuracy of the estimates; instead, it demonstrates the robustness of the causal-driven algorithm. Additionally, the deviation of \(\widetilde{T}^j_c\) on the \(x\)-axis is consistent with our previous observations. In this type of dataset, we should concentrate on the shift in the \textbf{joint distribution} of the time series rather than the \textbf{causal mechanism} shift.

\begin{figure}[h!]
    \centering
    \includegraphics[width=0.5\linewidth]{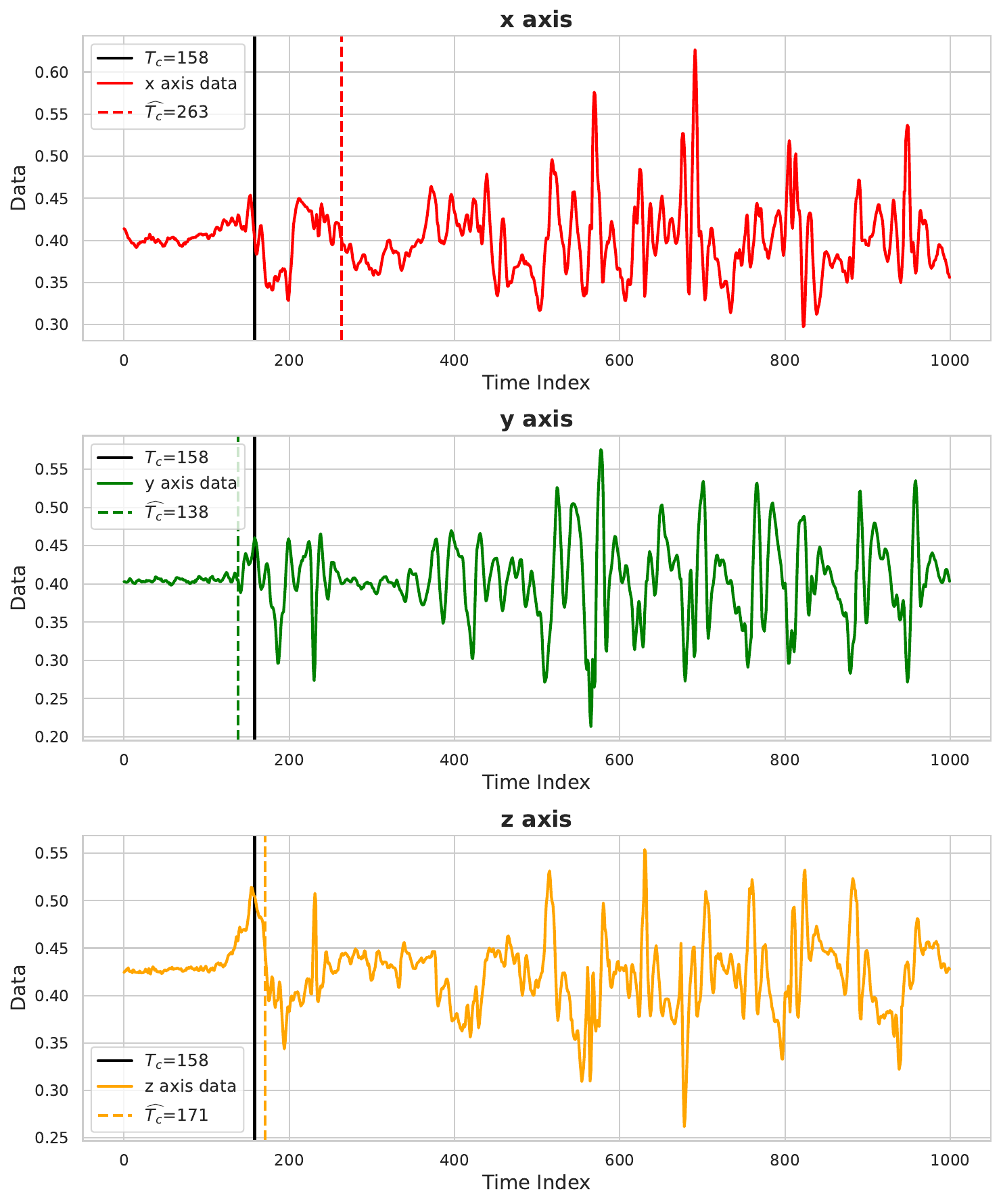}
    \caption{Causal-RuLSIF on human activity dataset. There are three subplots corresponding to the \(x\)-axis, \(y\)-axis, and \(z\)-axis, respectively. The black vertical line represents the true change point \(T^j_c\), while each colored vertical line indicates the estimated change point \(\widetilde{T}^j_c\). }
    \label{fig:human}
\end{figure}
\section{Limitations}\label{app:limitations}
Here we discuss the limitations of our proposed algorithm:

\textbf{Discrete-valued time series with infinite domain size and Continuous-valued time series}: 
If the domain size for the discrete data is large, a longer time series is required. Furthermore, the number of samples in each time series segment decreases exponentially as the domain size increases. This is a necessary trade-off if we want to avoid imposing model constraints on the distribution, such as linearity, and if we want to establish theoretical guarantees for non-IID time series data. The inherent limitation in sample efficiency arises because the conditional distribution is determined by the specific realization of a particular event.

Continuous-valued time series involves working with conditional probability density functions rather than conditional probability tables (CPTs). We have tried to extend this algorithm to probability density, but it is very challenging. For continuous-valued time series, the number of parent configurations is infinite, making it difficult to create a finite number of time series segments unless additional processing, such as categorizing the data or generating discrete analogs of a continuous distribution, is applied. 

\textbf{Sample Efficiency}: This limitation has been addressed to some extent, as discussed in Appendix~\ref{app:k-PC and Top-K}, but it remains a fundamental issue that we aim to explore in future work. In the current framework, without any special operation like k-PC and Top-K parents, the effective sample size of each time series segment, $|T_\text{sub}| \approx T/|s^{|\widehat{S\Pa}(X^j_t)|}|$, diminishes exponentially with domain size $s$ and parent size $\widehat{S\Pa}(X^j_t)$, in contrast to the baseline method, which maintains a sample size of $T$. The challenge lies in how to aggregate the information from each time series segment more efficiently.

\textbf{Smooth Causal Mechanism Change}:
With the \textit{Mechanism-shift} SCM, there is a constraint that two different mechanisms switch instantaneously. However, in some real-world scenarios, the causal mechanism shift may occur gradually over time.

\textbf{Causal Relations and Assumptions:} As mentioned in Section~\ref{app:case study}, before implementing our causal-driven algorithm, it is crucial to apply common sense or expert judgment to assess whether potential causal relationships exist. All the assumptions in Section~\ref{app:assumptions} should be met before using the algorithm.

\subsection{Experiments Under Assumption Violations}

In order to comprehensively evaluate the proposed algorithm under these limitations, we conducted experiments that violate Assumption \textbf{A6} Boundary Separation Assumption by placing change points near \( t = 1 \) and \( t = T \), specifically within 50 time points (Case B). We compared this with a standard setting (Case A), where change points are located within the interval \([50, T - 50]\), satisfying Assumption \textbf{A6}.

The algorithm's performance on continuous data after discretization varies depending on how the discretization process affects the underlying causal structure of the original data. To assess this, we also evaluated the algorithm on continuous data discretized using the 5-number summary (Case C).

The results are presented below. The table reports the average estimated error \(\frac{\widetilde{T}_c^j - T_c^j}{T_c^j}\) and its standard deviation over 50 random trials. In all experiments, we set \( T = 1500 \).

 \begin{table}[h]
\centering
\begin{tabular}{lcc|cc|cc}
\hline
 & \multicolumn{2}{c|}{Case A (normal)} & \multicolumn{2}{c|}{Case B (without A6)} & \multicolumn{2}{c}{Case C (discretization)} \\
 & Average & Std & Average & Std & Average & Std \\
\hline
 & 0.04 & 0.10 & 0.51 & 0.31 & 0.23 & 0.21 \\
\hline
\end{tabular}
\end{table}

\end{document}